\documentclass{article} 
\usepackage{arxiv/macros}
\usepackage{natbib} 
\title{A Cubic-regularized Policy Newton Algorithm\\ for Reinforcement Learning}

\author{
Mizhaan Prajit Maniyar \\
{\normalsize Indian Institute of Technology Madras} \\
{\normalsize \texttt{na17b014@smail.iitm.ac.in}}
\and
Akash Mondal\\
{\normalsize Indian Institute of Science} \\
{\normalsize \texttt{akashmondal@iisc.ac.in}}
\and 
Prashanth L. A.\\
{\normalsize Indian Institute of Technology Madras}\\
{\normalsize \texttt{prashla@cse.iitm.ac.in}}
\and
Shalabh Bhatnagar\\
{\normalsize Indian Institute of Science} \\
{\normalsize \texttt{shalabh@iisc.ac.in}}
}
\date{}
\begin{document}

\maketitle

\begin{abstract}
     We consider the problem of control in the setting of reinforcement learning (RL), where model information is not available. Policy gradient algorithms are a popular solution approach for this problem and are usually shown to converge to a stationary point of the value function. In this paper, we propose two policy Newton algorithms that incorporate cubic regularization. Both algorithms employ the likelihood ratio method to form estimates of the gradient and Hessian of the value function using sample trajectories. The first algorithm requires an exact solution of the cubic regularized problem in each iteration, while the second algorithm employs an efficient gradient descent-based approximation to the cubic regularized problem. We establish convergence of our proposed algorithms to a second-order stationary point (SOSP) of the value function, which results in the avoidance of traps in the form of saddle points. In particular, the sample complexity of our algorithms to find an $\epsilon$-SOSP is $O(\epsilon^{-3.5})$, which is an improvement over the state-of-the-art sample complexity of $O(\epsilon^{-4.5})$.
\end{abstract}

\section{Introduction}

Markov decision processes (MDPs)  provide a framework for analyzing sequential decision making problems under uncertainty. The aim here is to find a policy that optimizes a given performance objective, e.g., the discounted cumulative reward or cost. A direct solution approach for MDPs would require knowledge of the underlying transition dynamics. In practical settings, such information is seldom available, and one usually resorts to reinforcement learning (RL) algorithms \citep{sutton_book} that find optimal policies using sample trajectories. Due to its several advantages in high-dimensional action spaces and simplicity of their implementation, RL algorithms have been applied variously in problems such as natural language processing (\cite{nlp}), computer vision (\cite{cv}), speech recognition (\cite{speech}), and finance (\cite{finance}).

Classical RL algorithms based on lookup table representations suffer from the curse of dimensionality associated with large state spaces. A popular approach to overcoming this problem is through considering a prescribed parametric representation of policies and searching for the best policy within this class using a stochastic gradient (SG) algorithm. Policy gradient (PG) algorithms adopt this approach, and update the policy parameters using an estimate of the gradient of the expected sum of costs (or the value function) with respect to those parameters. The simplest such class of algorithms are trajectory-based methods that consider full Monte-Carlo returns for estimating the performance gradient such as REINFORCE \citep{Williams1992}. These algorithms work by increasing the probabilities of actions that lead to higher returns and thereby reduce the probabilities of actions that lead to lower returns using data sampled through interactions with the environment. The gradient of the value function for the given policy is estimated using the aforementioned Monte-Carlo returns.

		Incremental update algorithms such as actor-critic have been proposed as alternatives to trajectory-based methods \cite{sutton1999policy}. A common approach here to estimate the value function for any given policy is to introduce a critic recursion that does this estimation for any given parameter update and which in turn is updated using a `slower' actor recursion. Thus, actor-critic algorithms typically require two timescale recursions even though they are incremental-update procedures. Trajectory-based methods like REINFORCE involve a single-timescale in their update rule but typically suffer (like SGD) from high variance which leads to a high sample complexity for the algorithm. However, many works in recent times that involve trajectory-based methods have tried to address this high variance problem. For example, Trust Region Policy Optimization (TRPO) (\cite{TRPO}) can improve training stability by constraining the step length to be within a certain "trust region" and thereby obtain better sampling efficiency than the vanilla policy gradient approach. TRPO optimizes the loss function via the Kullback-Leibler (KL) divergence. In addition, Proximal Policy Optimization (PPO) (\cite{PPO}) is an improvement over TRPO as it employs objective function clipping and/or a penalty on the KL divergence with the trust region update that is compatible with SGD and simplifies the algorithm by eliminating the KL divergence. Both TRPO and PPO aim at solving constrained optimization problems involving inequality constraints.

Policy gradient algorithms and their analyses have received a lot of research attention recently, cf. \citep{fazel2018,agarwal2020,sutton1999policy,mohammadi2021,papini2018,vijayan2021smoothed,zhangK2020}. 
Since the value function is usually non-convex, the analysis of policy gradient algorithms usually establishes convergence to first order stationary points (FOSP) in the long run or to approximate stationary points in the non-asymptotic regime. These points also include unstable equilibria or traps such as local maxima and saddle points.
In the optimization literature, approaches to avoid traps either add extraneous noise in the gradient step \citep{chi17}, or show that the gradient estimates have omni-directional noise \citep{Akash}, or use second-order information in a stochastic Newton algorithm \citep{nesterov2006cubic,mokhtari2018newton}. 

In the context of RL, avoidance of traps has received very little attention and most of the previous works analyzing policy gradient algorithms show stationary convergence only. A few exceptions are \citep{konda1999actor,zhangK2020}. In \citep{konda1999actor}, the authors explore addition of extraneous isotropic noise to avoid traps in the context of a policy gradient algorithm. However,  their algorithm avoids traps under an additional condition that is hard to verify in typical RL settings. The latter condition requires the extraneous noise to dominate the martingale difference noise inherent to a policy gradient update. In \citep{zhangK2020}, the authors explore a different approach to avoid traps by using larger stepsize periodically. 

We propose a policy Newton algorithm that incorporates second-order information along with cubic regularization in the spirit of \cite{nesterov2006cubic}. Our algorithm has an improved sample complexity from incorporating Hessian information and this also leads us to avoid saddle points and converge to an $\epsilon$-SOSP.

We now summarize our contributions below.\\
{\textit{(a) Cubic-regularized policy Newton:}} For solving a finite horizon MDP, we propose a cubic regularized policy Newton (CR-PN) method that avoids saddle points and converges to an approximate second-order stationary point (SOSP), where the approximation is quantified by a parameter $\epsilon>0$. Such a point is referred to as $\epsilon$-SOSP. In this algorithm, we derive an estimate of the Hessian of the value function from sample episodes using Stein’s identity.
\\[0.5ex]
{\textit{(b) Approximate cubic-regularized policy Newton (ACR-PN):}} We  propose ACR-PN algorithm that finds an $\epsilon$-SOSP with high probability, while employing a computationally efficent solver for the cubic model in CR-PN. The latter approximate solver, borrowed from \cite{jordan2018stochastic}, avoids calculating the Hessian and instead uses the Hessian-vector product for computational simplicity. Though CR-PN provides an exact solution, but in a practical scenario, the computation of Hessian is very expensive due to the presence of a complex cost objective. In such cases, ACR-PN is very efficient.
\\[0.5ex]
{\textit{(c) Non-asymptotic convergence:}} We derive non-asymptotic bounds that quantify the convergence rate to $\epsilon$-SOSP for both algorithms. First, we prove that under assumptions common to the analysis of policy gradient algorithms, CR-PN converges to an $\epsilon$-SOSP within $O(\epsilon^{-1.5})$ iterations and with a sample complexity $O(\epsilon^{-3.5})$. These bounds hold both in expectation and with high probability. Second, we establish that ACR-PN converges to an $\epsilon$-SOSP with high probability within a number of iterations and with a sample complexity that is comparable to that of CR-PN. 
These rates significantly improve upon the sample complexity of existing algorithms in the literature. For instance, a REINFORCE-type algorithm has $O(\epsilon^{-4})$ sample complexity, and the Hessian-aided policy gradient method (\cite{sosp}) has $O(\epsilon^{-4.5})$ complexity.

\paragraph{Related work.}
Policy gradient algorithms and their analyses have received a lot of research attention, cf. \citep{fazel2018,agarwal2020,sutton1999policy, bhatnagar2009natural,  mohammadi2021,papini2018,vijayan2021smoothed,zhangK2020}. 
In \citep{furmston2016approx}  the authors propose policy Newton algorithms for solving an MDP in a setting where the model (or the transition dynamics) is known. In \citep{shen2019hessian}, the authors propose a policy gradient algorithm that incorporates second-order information and establish convergence to an approximate stationary point. In \cite{VBAC}, an actor-critic algorithm for constrained MDPs is developed for average cost MDPs in the full-state case. In \citep{SBKLAC}, a function approximation based actor-critic algorithm for average cost constrained MDPs is developed that does policy gradient on the Lagrangian along the actor update and involves a temporal-difference TD($\lambda$) critic. In \citep{SBAC}, the discounted cost version of the algorithm in \cite{SBKLAC} is developed except that a data-driven gradient estimation procedure \cite{shalabh_book} is adopted. 

A preliminary version of our work was part of the first author's master's thesis, see \cite{mizhaan}, before a concurrent work, see \cite{WANG22}. In the aforementioned reference, the authors propose an approximate cubic-regularized Newton algorithm and provide sample complexity bounds. 
In contrast, we analyze both the exact and approximate cubic-regularized Newton algorithm variants. For the former, we derive bounds using a proof technique that is radically different from that employed in \cite{WANG22}. 
Moreover, our analysis for the approximate Newton variant is a lot simpler, as we invoke the bounds in \cite{jordan2018stochastic} after verifying the necessary assumptions, whereas the analysis in \cite{WANG22} mimics the proof in \cite{jordan2018stochastic}.

The rest of the paper is organized as follows:
Section \ref{sec:PG_framework} describes the problem formulation and the policy gradient framework. Section \ref{sec:algorithm} presents the cubic-regularized policy Newton algorithm, while  Section \ref{sec:convergence} established non-asymptotic bounds for convergence to an approximate second-order stationary point of the objective.
Section \ref{sec:approx} describes the variant of the cubic-regularized policy Newton algorithm, which incorporates an approximate solution scheme for the cubic-regularized problem. Section \ref{sec:proofs} provides detailed proofs of convergence. Finally, Section \ref{conclusion} provides the concluding remarks.

\section{Policy gradient framework}
\label{sec:PG_framework}

A Markov decision process (MDP) is a tuple of the form $(\mathcal{S}, \mathcal{A}, P, k, \gamma, \rho)$ where $\mathcal{S}$ is the state space; $\mathcal{A}$ is the action space; $P(s'|s,a)$ represents the underlying transition dynamics that governs the state evolution from $s$ to $s'$ under a given action $a\in\mathcal{A}$ of the MDP agent; $c(s,a)$ denotes the single-stage cost accumulated by the agent in state $s$ on taking action $a$; $\gamma$ is the discount factor; and $\rho(s_0)$ is the distribution of the starting state $s_1$. Let the probability of transition into a state $s_{h+1}$ at instant $h+1$ given that the state of the environment at instant  $h$ is $s_h \in \mathcal{S}$ and the action chosen by the agent therein is $a_h \in \mathcal{A}$ be $P(s_{h+1} | s_{h}, a_h)$. 
The actions are chosen according to a probability distribution $\pi(a_h | s_h)$, which is conditioned over the current state. We shall call $\pi$ as the policy used by the agent to select actions.
We consider a finite horizon discounted setting. 
In an episodic task, we denote the trajectory of states and actions until termination at instant $H$ as  $\tau := (s_0, a_0, \ldots ,a_{H-1}, s_H)$, where $s_0 \sim \rho(s_0)$ and $H$ is the trajectory horizon or episode
length. The probability of trajectory $\tau$ following a policy $\pi$ is given by
\begin{equation}\label{traj}
\begin{split}
    p(\tau ; \pi) &:=  \left( \prod_{h = 0}^{H-1} P(s_{h+1} | s_h , a_h) \pi (a_h | s_h) \right)\rho(s_0) .
\end{split}
\end{equation}
We denote the discounted cumulative cost for a trajectory $\tau$ with a discount factor $\gamma \leq 1$ as \\$\mathcal{G}(\tau) := \sum_{h = 0}^{H-1} \gamma^{h-1} c(s_h, a_h)$. Our objective is to minimize the \textit{expected} discounted cumulative cost given by
\begin{align}
    J(\pi) &:= \E[\tau \sim p(\tau ; \pi)]{\mathcal{G}(\tau)} 
    =  \E[\tau \sim p(\tau ; \pi)] {\sum_{h=0}^{H-1} \gamma^{h-1} c(s_h, a_h) }.\label{J}
\end{align}
We assume that the policy is parameterized by a vector $\theta \in \mathbb{R}^d$ and use the notation $\pi_{\theta}$ as a shorthand for the distribution $\pi(a_h | s_h; \theta)$. 
Also, we denote $p(\tau ; \theta) = p(\tau ; \pi_{\theta})$  and $J(\theta) = J(\pi)$ as they are both conditioned on the same information. For simplicity, we assume here that the terminal cost is $0$. 

Our aim is to find the policy that minimizes the above objective, i.e., to find 
\begin{align}\label{obj}
    \theta^* &\in \argmin_{\theta \in \mathbb{R}^d} J(\theta),
\end{align}
where $\theta^*$ is an \textit{optimal policy} parameter. 

The gradient $\nabla J(\theta)$ of the expected cost $J(\theta)$ can be written as
\begin{equation}
    \nabla J(\theta) = \sum_{h=0}^{H-1} \sum_{\tau_h} \gamma^{h-1} c(s_h, a_h) \nabla p(\tau_h ; \theta) \,.
\end{equation}

Using now the fact that $\nabla p(\tau_h ; \theta) = p(\tau_h ; \theta)\nabla \log p(\tau_h ; \theta)$, we obtain
\begin{flalign}
    \nabla J(\theta)&= \sum_{h=0}^{H-1} \sum_{\tau_h} \gamma^{h-1} c(s_h, a_h) \nabla  \log p(\tau_h ; \theta) \, p(\tau_h ; \theta) \, \\
    &= \sum_{h=0}^{H-1} \E[\tau_h \sim p(\tau_h ; \theta)] {\gamma^{h-1} c(s_h, a_h) \nabla  \log p(\tau_h ; \theta) } \\
    & \stackrel{\ref{traj}}{=} \sum_{h=0}^{H-1} \sum_{i=0}^h  \E[\tau \sim p(\tau ; \theta)] {\gamma^{h-1} c(s_h, a_h)  \nabla \log \pi (a_i | s_i ; \theta) } ,
\end{flalign}
where in the last equality we use the fact that $ \gamma^{h-1} c(s_h, a_h)  \nabla \log \pi (a_i | s_i ; \theta) $ with $i \le h$ is independent of the randomness after $a_h$ has been chosen. Since the computation of the above expectation requires averaging over trajectories $\tau\sim p(\tau;\theta)$, the above expectation is infeasible to compute in practice. Thus one may sample $m_k$ trajectories to obtain an estimate of the actual gradient, which has the following form:
\begin{equation}
    \begin{split}
        \hat{\nabla}J(\theta)
        &= \dfrac{1}{m_k} \sum\limits_{j=1}^{m_k} \sum\limits_{h=0}^{H-1} \sum_{i=0}^h {\gamma^{h-1} c(s_h^j, a_h^j)  \nabla \log \pi (a_i^j | s_i^j ; \theta) }.
    \end{split}
\end{equation}
The above estimate is used in the well-known REINFORCE policy gradient algorithm \cite{Williams1992}, where $s_x^j,a_x^j$ for $x\in[h,i]$ are state-action pair for $j$-th trajectory.

In this work, our focus is on estimating \eqref{obj} by employing a second-order (Hessian-based) estimate instead of only a gradient estimate. For this purpose, we require an expression for the Hessian of the objective, in addition to the policy gradient given above. Such an expression has been derived earlier in \cite{furmston2016approx}. For the sake of completeness, we specify this expression and also provide a proof in the Appendix.

For deriving the policy gradient and Hessian expressions, we require assumptions on the regularity of the MDP and the smoothness of our parameterized policy $\pi_{\theta}$. These assumptions are specified below.

\begin{assumption}[Bounded costs]
\label{ass:bddRewards}
The absolute value of the cost function of the MDP is bounded, i.e., $\exists K\in (0,\infty)$ such that 
\begin{equation}
    |c(s, a)| \le K , \qquad \forall (s, a) \in \mathcal{S} \times \mathcal{A}.
\end{equation}
\end{assumption}

\begin{assumption}[Parameterization regularity]\label{ass:paramreg}
For any choice of the parameter $\theta$, any state-action pair $(s, a)$, there exist constants $0< G,L_1< \infty$ such that
\begin{equation*}
    \norm{\nabla \log \pi(a|s ; \theta)} \le G \quad \textrm{and} \quad \norm{\nabla^2 \log \pi(a|s ; \theta) } \le L_1 .
\end{equation*}
\end{assumption}
\begin{assumption}[Lipschitz Hessian]\label{ass:liphesspolicy}
For any pair of parameters $(\theta_1, \theta_2)$, and any state-action pair $(s, a)$, there exists a constant $L_2$ such that
\begin{gather*}
    \norm{\nabla^2 \log \pi(a|s ; \theta_1) - \nabla^2 \log \pi(a|s ; \theta_2)} \le L_2 \norm{\theta_1 - \theta_2}.
\end{gather*}
\end{assumption}
\noindent
Note that \ref{ass:bddRewards} and \ref{ass:paramreg} are standard in the literature on policy  gradient and actor-critic algorithms as shown in \cite{shen2019hessian} and $\norm \cdot$ denotes the $l_2$ norm for vectors and the operator norm for matrices. Furthermore, \ref{ass:liphesspolicy} is also standard in second-order policy search algorithms, cf. \cite{zhangK2020}. 

The result below establishes that the expected cost objective is smooth, and its gradient, as well as Hessian, are well defined and Lipschitz continuous.
\begin{proposition}\label{lipschitz-prop}
Under \ref{ass:bddRewards}--\ref{ass:liphesspolicy}, for any  $\theta_1, \theta_2\in \mathbb{R}^d$, we have
\begin{align}
    \norm{J(\theta_1) - J(\theta_2)} &\le M_{\Hess} \norm{\theta_1 - \theta_2},\nonumber\\ 
    \norm{\nabla J(\theta_1) - \nabla J(\theta_2)} &\le G_{\Hess} \norm{\theta_1 - \theta_2}, \textrm{ and }\nonumber\\ 
    \norm{\nabla^2 J(\theta_1) - \nabla^2 J(\theta_2)} &\le L_{\Hess} \norm{\theta_1 - \theta_2},  \label{eq:ghlh}
    \end{align} 
    where $M_{\Hess} := KGH^3, G_{\Hess} := H^3 G^2 K + L_1K H^2$  and $L_{\Hess} := H^4 G^3 K + 3H^3 G L_1 K + L_2KH^2$.
\end{proposition}
\begin{proof}
See Lemmas \ref{lipgrad}--\ref{liphess} and their proofs in Section \ref{pf:lipschitz-prop}.
\end{proof}

 We now present the policy gradient and Hessian theorem.
\begin{theorem}[\textbf{\textit{Policy gradient and Hessian theorem}}]
\label{thm:policyGradAndHessian}

Let 
\begin{align}
\Psi_{i}(\tau) &:= \sum_{h=i}^{H-1} \gamma^{h-1} c(s_h, a_h) \textrm{ and }
\Phi(\theta ; \tau) := \sum_{i=0}^{H-1} \Psi_{i}(\tau) \log \pi(a_i | s_i ; \theta).\label{eq:phi}
\end{align}
Then, under  assumptions \ref{ass:bddRewards}-\ref{ass:liphesspolicy}, the gradient $\nabla J(\theta)$ and the Hessian $\nabla^2 J(\theta)$ of the objective \eqref{J} are given by 
\begin{align}
\label{gradJ}
    \nabla J(\theta) &=   \E[\tau \sim p(\tau ; \theta)] {\nabla \Phi(\theta ; \tau)},\\
    \nabla^2 J(\theta)\! &=\!   \E[\tau \sim p(\tau ; \theta)]{\nabla \Phi(\theta ; \tau) \nabla^{\top} \log p(\tau ;\theta) \!+\! \nabla^2 \Phi(\theta ; \tau) }.\label{hessJ}
\end{align}
\end{theorem}
\begin{proof}
Refer to Appendix \ref{pf:policyGradAndHessian}.
\end{proof}

\section{Cubic-regularized policy Newton algorithm}
\label{sec:algorithm}


A stochastic gradient algorithm to find a local optimum of the objective function in the problem \eqref{obj} would perform an incremental update of the policy parameter as follows:

\begin{equation}
    \theta_{k+1} = \theta_k - \eta M(\theta_k) \nabla J(\theta_k),
\end{equation}
where $\eta \in \mathbb{R}^{+}$ is the step size and $M(\theta)$ is a preconditioning matrix that could depend on the policy parameter $\theta$. 

If $J$ is smooth and $M(\theta)$ is positive-definite, then the policy parameter update ensures a decrease in the objective, viz., the total 
expected cost for sufficiently small $\eta$. 
Note that if $M(\theta)$ is the identity matrix, then the update rule above corresponds to a gradient step, while $M(\theta)= \nabla^2 J(\theta)^{-1}$ would result in a Newton step. 

In a typical RL setting, it is not feasible to find the exact gradient or Hessian of the objective function since the underlying transition dynamics of the environment are unknown. Instead, one has to form sample-based estimates of these quantities. Now, if we use an estimate of the Hessian in place of the preconditioning matrix, we cannot assure a stable gradient descent as the estimate may not be positive-definite as required. This makes the classical Newton update a bad candidate for our policy search algorithm. \cite{np06} motivates an algorithm called the cubic-regularized Newton method in a deterministic setting which tackles these issues and more. They show that the standard Newton step ($\eta=1$) can alternatively be presented as follows:
\begin{align*}
       &\theta_{k+1}=
         \argmin_{\theta \in \mathbb{R}^d} \left\{ \!\innerproduct{\nabla J(\theta_k)}{\theta \! - \!\theta_k} \!+\! \frac{1}{2} \innerproduct{\nabla^2 J(\theta_k) (\theta \! - \!\theta_k)}{\theta \! - \!\theta_k} \right\}. 
\end{align*}
The cubic regularized Newton step adds a cubic term to the auxiliary function in the following manner:
\begin{align}
    &\theta_{k+1} = \argmin_{\theta  \in \mathbb{R}^d} \left\{  \!\innerproduct{\nabla J(\theta_k)}{\theta\!  - \!\theta_k}\! +\! \frac{1}{2} \innerproduct{\nabla^2 J(\theta_k) (\theta\!  - \!\theta_k)}{\theta\!  - \!\theta_k}\! +\! \frac{\alpha}{6} \norm{\theta - \theta_k}^3\right\}, 
\end{align}
where $\alpha \in \mathbb{R}^{+}$ is the regularization parameter. 
In a general stochastic optimization setting, using the cubic-regularized Newton step, the authors in \cite{kkumar2018zeroth} establish convergence to local minima. We adopt a similar approach in a RL setting to propose/analyze a cubic-regularized policy Newton method with gradient/Hessian estimates derived using the result in Theorem \ref{thm:policyGradAndHessian}.
Algorithm \ref{alg:policyNewton} presents the pseudocode for the cubic-regularized policy Newton algorithm with a gradient and Hessian estimation scheme that is described below.  

From Theorem \ref{thm:policyGradAndHessian}, the policy gradient and Hessian can be seen as expectations of $g$ and $\Hess$ defined below.
\begin{equation}
    \begin{split}
        g(\theta ; \tau) &:= \nabla \Phi(\theta ; \tau),\\
    \Hess(\theta ; \tau) &:= \nabla \Phi(\theta ; \tau) \nabla^{\top} \log p(\tau ;\theta) + \nabla^2 \Phi(\theta ; \tau) .
    \end{split}
\label{eq:gradHessEstOneSample}
\end{equation}
The above estimates are calculated by the information obtained from a given trajectory $\tau$ and policy parameter $\theta$. We simulate multiple trajectories and calculate these estimates for each of them and then take their averages to obtain the final such estimates.

 \begin{algorithm}[!h]
 \caption{Cubic-regularized policy Newton (CR-PN)}\label{alg:policyNewton}
 \SetKwInOut{Init}{Input}\SetKwInOut{Output}{Output}
 \Init{Initial parameter $\theta_0 \in \mathbb{R}^d$, a non-negative sequence $\{ \alpha_k \}$, positive integer sequences $\{ m_k \} \textrm{ and } \{ b_k \}$, and an iteration limit $N \ge 1$.}
 \For{$k = 1, \ldots, N$}{
  \tcc{Monte Carlo simulation}
     Simulate $\min \{ m_k, b_k \}$ number of trajectories according to $\theta_{k-1}$, randomly pick $m_k$ trajectories for set $\mathcal{T}_m$ and $b_k$ trajectories for set $\mathcal{T}_b$\;
     \tcc{Gradient estimation}
     $
         \Bar{g}_k 
         = \dfrac{1}{m_k} \sum\limits_{\tau \in \mathcal{T}_m} \sum\limits_{h=0}^{H-1} \Psi_h(\tau) \nabla \log \pi (a_h | s_h ; \theta_{k-1})$\\
     where the state-action pairs $ ( s_h, a_h )$ belong to the respective trajectories $\tau$\;
    
     \tcc{Hessian estimation}
     $
         \Bar{\Hess}_k 
         = \dfrac{1}{b_k} \sum\limits_{\tau \in \mathcal{T}_b} \Big( \sum\limits_{h=0}^{H-1} \Psi_h(\tau) \nabla \log \pi (a_h | s_h ; \theta_{k-1}) \sum_{h'=0}^{H-1} \nabla^{\top} \log \pi (a_{h'} | s_{h'} ; \theta_{k-1}) \Big)$ \\
         $\quad\qquad +  \dfrac{1}{b_k} \sum_{\tau \in \mathcal{T}_b} \sum_{h=0}^{H-1} \Psi_h(\tau) \nabla^2 \log \pi (a_h | s_h ; \theta_{k-1})$;
        
     \tcc{Policy update (cubic regularized Newton step) }
     Compute
     \begin{align}
         \theta_{k} &= \argmin_{\theta \in \mathbb{R}^d} \left\{ \Tilde{J}^k(\theta) \equiv \Tilde{J}(\theta, \theta_{k-1}, \Bar{\Hess}_k, \Bar{g}_k, \alpha_k) \right\} , \textrm{ where }\nonumber\\
\label{aux}
             \Tilde{J}(x, y, \Hess, g, \alpha) &= \innerproduct{g}{x-y} + \frac{1}{2} \innerproduct{\Hess (x-y)}{x-y} + \frac{\alpha}{6} \norm{x-y}^3.
     \end{align}
 }
 \Output{Policy $\theta_N$} 
 \end{algorithm}

The result that we state below provides bounds on the single trajectory-based gradient and Hessian estimates defined in \eqref{eq:gradHessEstOneSample}. These bounds are necessary for establishing convergence to a local minima of the objective $J$ with high probability (see Theorem \ref{thm:High-probability bound} below). 
\begin{lemma}
	\label{lemma:gradHessBounds}
	Let $g(\theta ; \tau), 
	\Hess(\theta ; \tau)$ be the gradient and Hessian estimates formed using \eqref{eq:gradHessEstOneSample}, respectively. Then, under \ref{ass:bddRewards} \ref{ass:paramreg} and \ref{ass:liphesspolicy} for any parameter $\theta$ and trajectory $\tau$, we have almost surely
 \begin{align*}
		\norm{g(\theta ; \tau)\!-\!\nabla J(\theta)} \le M_1\textit{ and } \norm{\Hess(\theta ; \tau)\!-\!\nabla^2\! J(\theta)} \le M_2,
	\end{align*}
	where $M_1:=GKH^2(H + 1),$ and $M_2:=2G_{\Hess}.$ The constants $K, G$ are given in Assumptions \ref{ass:bddRewards} and \ref{ass:liphesspolicy}, respectively, $H$ is the horizon, and $G_\Hess$ is defined in Proposition \ref{lipschitz-prop}.
\end{lemma}
\begin{proof}
	See Appendix \ref{pf:gradHessBounds}.
\end{proof}

\section{Main results}
\label{sec:convergence}
In this section, we first define a first and second-order stationary point of the expected cost objective. Subsequently, we prove that Algorithm \ref{alg:policyNewton} converges to a second-order stationary point.

\begin{definition}[\textbf{$\epsilon$-first-order stationary point}]
Fix $\epsilon>0$. Let $\theta_R$ be the random output of an algorithm for solving \eqref{obj}. Then, $\theta_R$ is an $\epsilon$ first-order stationary point ($\epsilon$-FOSP) of $J$ if
\begin{align}\label{SOSP}
{\E{\norm{\nabla J(\theta_R)}}} \le {\epsilon}.
\end{align}
\end{definition}
The point defined above is an approximation to a first-order stationary point where the gradient vanishes, i.e., $\epsilon=0$. This could potentially be a saddle point. In order to verify whether it is optimum, we need information with regard to the curvature of the underlying objective. This motivates us to our second definition:
\begin{definition}[\textbf{$\epsilon$-second-order stationary point}]
Fix $\epsilon>0$. Let $\theta_R$  be the random output of an algorithm for solving \eqref{obj}. Then, for some $\rho>0$, $\theta_R$ is an $\epsilon$ second-order stationary point ($\epsilon$-SOSP) in expectation if
\begin{align}
\max \left\{ \sqrt{\E{\norm{\nabla J(\theta_R)}}}, \frac{-1}{\sqrt{\rho}} \E{\lambda_{\min} \left( \nabla^2 J(\theta_R) \right) } \right\} \le \sqrt{\epsilon},
\end{align}
where  $\lambda_{\min} (\cdot)$ denotes the minimum eigenvalue of the given matrix. 

Furthermore, for any $\epsilon>0$ $\theta_R$ is said to an $\epsilon$-SOSP with high probability if the following bound holds with probability $1-\delta$ for any  $\delta\in (0,1)$:
\begin{align}
\max \left\{ \sqrt{{\norm{\nabla J(\theta_R)}}}, \frac{-1}{\sqrt{\rho}} \lambda_{\min} \left( \nabla^2 J(\theta_R) \right)  \right\} \le \sqrt{\epsilon},
\end{align}
\end{definition}
In the above definition, if $\epsilon = 0$, then $\theta_R$ is a second-order stationary point. Therefore, a second-order stationary point is where the gradient is zero, and the Hessian is positive semi-definite. Such definitions are standard in second-order optimization literature, and an algorithm that outputs an $\epsilon$-SOSP approximates the local minimum better than one that outputs $\epsilon$-FOSP, cf. \citep{kkumar2018zeroth, jordan2018stochastic}. 

We now state the result that establishes convergence of Algorithm \ref{alg:policyNewton} to an $\epsilon$-SOSP of the objective \eqref{J} in expectation.
\begin{theorem}[\textit{\textbf{Bound in expectation}}]
\label{thm:cubRegNewtonBound}
Let $ \{\theta_1, \dots, \theta_N\} $ be computed by Algorithm \ref{alg:policyNewton} with the following parameters:
\begin{align}\label{params}
    \alpha_k = 3 L_{\Hess},  N = \frac{12\sqrt{L_{\Hess}} (J(\theta_0)-J^*)}{\epsilon^{\frac{3}{2}}},m_k = \frac{25 G_g^2}{4 \epsilon^2} ,  
     b_k = \frac{36 \sqrt[3]{30 (1 + 2 \log 2d)} d^{\frac{2}{3}} G_{\Hess}^2}{\epsilon}.
\end{align}
Let $\theta_R$ be picked uniformly at random from $\{\theta_1, \ldots, \theta_N\}$.
Then, we have
\begin{flalign}\label{main}
       & 5 \sqrt{\epsilon} \ge \max \left\{ \sqrt{\E{\norm{\nabla J(\theta_R)}}} , \frac{-5}{7 \sqrt{L_{\Hess}}} \E{\lambda_{\min} \left( \nabla^2 J(\theta_R) \right) } \right\} .
\end{flalign}
where $G_{\Hess}$ and $L_{\Hess}$ are defined as in \eqref{eq:ghlh}.
\end{theorem}
\begin{proof}
See Appendix \ref{pf:cubRegNewtonBound} for a detailed proof.
\end{proof}
A few remarks are in order.
\begin{remark}
Since $(J(\theta_0)-J^*)$ is unknown in a typical RL setting, to aid practical implementations, one could choose $N = \frac{24 K H\sqrt{L_{\Hess}} }{\epsilon^{\frac{3}{2}}}$, and the bound in \eqref{main} would continue to hold, since $(J(\theta_0)-J^*) \le 2KH$.
\end{remark}
\begin{remark}
As a consequence of Theorem \ref{thm:cubRegNewtonBound}, to obtain an $\epsilon$-SOSP of the problem, the total number of trajectories required to compute the gradient and the Hessian are bounded by
$O \left( \frac{1}{\epsilon^{\frac{7}{2}}} \right)$ and $O \left( \frac{d^{\frac{2}{3}}}{\epsilon^{\frac{5}{2}}} \right)$, respectively. This is of a higher order in contrast to the HAPG algorithm proposed by \cite{shen2019hessian}, which requires $ O \left( \frac{1}{\epsilon^3} \right)$ number of trajectories. However, the total number of time steps that it requires for our algorithm to converge is $O \left( \frac{1}{\epsilon^{1.5}} \right)$ versus the $O \left( \frac{1}{\epsilon^{2}} \right)$ required for HAPG. Furthermore, our algorithm ensures convergence to an $\epsilon$-SOSP thereby avoiding saddle points, while HAPG is shown to converge to an $\epsilon$-FOSP, which could potentially be a  trap (e.g. saddle point).
\end{remark}

Next, we establish convergence of Algorithm \ref{alg:policyNewton} to an $\epsilon$-SOSP with high probability.
\begin{theorem}[\textit{\textbf{High-probability bound}}]
\label{thm:High-probability bound}
Let $ \{\theta_1, \dots, \theta_N\} $ be computed by Algorithm \ref{alg:policyNewton} with the following parameters:
\begin{equation}\label{params1}
    \begin{split}
        \alpha_k &= 3 L_{\Hess},  N = \frac{12\sqrt{L_{\Hess}}(J(\theta_0)-J^*)}{\epsilon^{\frac{3}{2}}},
         m_k = \max\bigg(\frac{M_1}{t},\frac{M_1^2}{t^2}\bigg)\frac{8}{3}\log\frac{2d}{\delta'} ,
         b_k = \max\bigg(\frac{M_2}{\sqrt{t_1}},\frac{M_2^2}{t_1}\bigg)\frac{8}{3}\log\frac{2d}{\delta'},
    \end{split}
\end{equation}
where $t,t_1$ are any positive constants and $\delta'\in (0,1)$. Let $\theta_R$ be picked uniformly at random from $\{\theta_1, \ldots, \theta_N\}$.
Then, with probability $1-2N\delta'$ we have under the condition of Lemma \ref{lemma:gradHessBounds}
\begin{align}\label{main1}
    5 \sqrt{\epsilon} \ge \max \left\{ \sqrt{{\norm{\nabla J(\theta_R)}}} , \frac{-5}{7 \sqrt{L_{\Hess}}} {\lambda_{\min} \left( \nabla^2 J(\theta_R) \right) } \right\} ,
\end{align}
 where $G_{\Hess}$ and $L_{\Hess}$ are defined as in \eqref{eq:ghlh}. 
\end{theorem}
\begin{proof}
See Appendix \ref{pf:High-probability bound}.
\end{proof}
From the bound above, it is apparent that  Algorithm \ref{alg:policyNewton} will output an $\epsilon$-SOSP  with probability at least $1-2\delta'N$ within $O \left(\frac{1}{\epsilon^{1.5}}\right)$ number of iterations.
 Further, as in the case of the expectation bound in Theorem \ref{thm:cubRegNewtonBound}, the total number of trajectories required to estimate the gradient and Hessian are bounded by
\begin{align*}
    \sum_{k=1}^N m_k = O \left( \frac{1}{\epsilon^{\frac{7}{2}}} \right) , \textrm{ and }
    \sum_{k=1}^N b_k = O \left( \frac{1}{\epsilon^{\frac{5}{2}}} \right),
\end{align*}
respectively.

\section{Approximate cubic-regularized policy Newton}
\label{sec:approx}
The cubic regularized policy Newton algorithm, which is described in Algorithm \ref{alg:policyNewton}, requires an exact solution to the following optimization in each iteration:
\begin{equation}
    \begin{split}
        \theta_{k+1}=\argmin_\theta \left\{J(\theta_k)+\innerproduct{\Bar{g}_k}{\theta  - \theta_k} + \frac{1}{2} \innerproduct{\Bar{\Hess}_k (\theta  - \theta_k)}{\theta  - \theta_k} + \frac{\rho}{6} \norm{\theta - \theta_k}^3\right\}.
    \end{split}
\label{eq:cub-reg-pb}
\end{equation}
In practice, it may not always be possible to obtain an exact solution to the problem above, while one can perform a few gradient descent steps to arrive at an approximate solution that may be `good enough'. Such an approach has been explored in a general stochastic optimization context in \cite{jordan2018stochastic}. 

We use the algorithm in \cite{jordan2018stochastic} as a blackbox to arrive at a bound for a variant of Algorithm \ref{alg:policyNewton} that solves the cubic-regularized problem approximately. For the sake of  completeness, the approximate policy Newton algorithm is presented in Algorithm \ref{alg:approx_policyNewton}, and this algorithm is an instantiation of the template from \cite{jordan2018stochastic} with gradient and Hessian estimates along the lines of those employed in Algorithm \ref{alg:policyNewton}.

\begin{algorithm}[!h]
	\caption{Approximate cubic-regularized policy Newton (ACR-PN)}\label{alg:approx_policyNewton}
	\SetKwInOut{Input}{Input}\SetKwInOut{Output}{Output}
	\Input{ mini-batch sizes $m_k,b_k$, initialization $\theta_0$, number of iterations $
 \bar N$, and final tolerance $\epsilon$.} 
	\For{$k=0$ {\bfseries to} $\bar N$}{
		  \tcc{Monte Carlo simulation}
		Simulate $\min \{ m_k, b_k \}$ number of trajectories according to $\theta_{k-1}$, randomly pick $m_k$ trajectories for set $\mathcal{T}_m$ and $b_k$ trajectories for set $\mathcal{T}_b$.\;
\tcc{Gradient estimation}
		
$			\Bar{g}_k 
			= \dfrac{1}{m_k} \sum\limits_{\tau \in \mathcal{T}_m} \sum\limits_{h=0}^{H-1} \Psi_h(\tau) \nabla \log \pi (a_h | s_h ; \theta_{k-1})$\;
			
			\tcc{Hessian estimation}
			$
			\Bar{\Hess}_k 
			= \dfrac{1}{b_k} \sum\limits_{\tau \in \mathcal{T}_b} \Big( \sum\limits_{h=0}^{H-1} \Psi_h(\tau) \nabla \log \pi (a_h | s_h ; \theta_{k-1})  \sum_{h'=0}^{H-1} \nabla^{\top} \log \pi (a_{h'} | s_{h'} ; \theta_{k-1}) \Big)$ \\
			$\quad\qquad +  \dfrac{1}{b_k} \sum_{\tau \in \mathcal{T}_b} \sum_{h=0}^{H-1} \Psi_h(\tau) \nabla^2 \log \pi (a_h | s_h ; \theta_{k-1})$\;
			
		\tcc{$\Delta,\delta_J \leftarrow$ Cubic-Subsolver($\Bar{g}_k,\Bar{\Hess}_k[\cdot],\epsilon$)}\label{PN-CS}
  \If{$\norm{\Bar{g}}\geq\frac{l^2}{\rho}$}{
    $R_c\leftarrow -\frac{\Bar{g}\Bar{\Hess}[\Bar{g}]}{\rho\norm{g}^2}+\sqrt{\Big(\frac{\Bar{g}^T\Bar{\Hess}[\Bar{g}_k]}{\rho\norm{\Bar{g}^2}}\Big)^2+\frac{2\norm{g}}{\rho}}$\\
   $\Delta\leftarrow-R_c\frac{\Bar{g}}{\norm{g}}$
   
   }
   \Else{
    $\Delta\leftarrow 0, \sigma\leftarrow c'\frac{\sqrt{\epsilon\rho}}{l},\eta\leftarrow\frac{1}{20l}$\\
   $\Tilde{g}\leftarrow \Bar{g}+\sigma\zeta \textrm{ for } \zeta\sim\textrm{ Unif }(\mathbb{S}^{d-1})$\\
  \For{$k=0$ {\bfseries to} $\hat{N}$}{$\Delta\leftarrow\Delta-\eta(\Tilde{g}+\Bar{\Hess}[\Delta]+\frac{\rho}{2}\norm{\Delta}\Delta)$}
  
   }
$\theta_{k+1}\leftarrow\theta_k+\Delta$\\
		\If{$\delta_J\geq -\frac{1}
			{100}\sqrt{\frac{\epsilon^3}{\rho}}$}
   {\tcc{$\Delta\leftarrow$ Cubic-Finalsolver($\Bar{g}_k,\Bar{\Hess}_k[\cdot],\epsilon$)}
   $\Delta\leftarrow 0, \Bar{g}_J\leftarrow\Bar{g},\eta\leftarrow\frac{1}{20l}$\;
      \While{$\norm{\Bar{g}_J}\geq\frac{\epsilon}{2}$}
      { 
      	$\Delta\leftarrow\Delta-\eta\Bar{g}_J$\;
      $\Bar{g}_J\leftarrow \Bar{g}_J +\Bar{\Hess}[\Delta]+\frac{\rho}{2}\norm{\Delta}\Delta$\;
  }
			$\theta^*\leftarrow\theta_k+\Delta$}
	}
	{\bfseries Output:} $\theta^*$ if the early termination condition was reached, otherwise the final iterate $\theta_{\bar N}$
\end{algorithm}
The ACR-PN algorithm uses the gradient and Hessian estimates to solve the cubic-regularized problem \eqref{eq:cub-reg-pb} using a gradient descent-type algorithm. In particular, the ``Cubic-Subsolver'' routine returns the parameter change $\Delta$, which is used to update the policy parameter $\theta_k$. 
If the corresponding change in $\Bar{J}_k(\theta)$, i.e., $\delta_J:=\Bar{J}_k(\theta_k+\Delta)-\Bar{J}_k(\theta_k)$ satisfies a certain stopping criteria, then 
ACR-PN calls the "Cubic-Finalsolver'' subroutine to perform several gradient descent steps so that the cubic-regularized problem \eqref{eq:cub-reg-pb} is solved accurately. As an aside, we note that the ACR-PN uses the Hessian estimate only through Hessian-vector products. The reader is referred to \cite{jordan2018stochastic} for the details of the two subroutines mentioned above.

We now turn to establishing convergence of Algorithm \ref{alg:approx_policyNewton} to an $\epsilon$-SOSP with high probability.  The main claim of Theorem \ref{thm:approxNewtonBound} would follow from Theorem 1 of  \cite{jordan2018stochastic}, provided we verify Assumptions 1 and 2 from the aforementioned reference. For the sake of completeness, we state these assumptions and Theorem 1 of \cite{jordan2018stochastic} below.\\
\textbf{Assumption 1}: The function $J$ has the following property,
\begin{itemize}
	\item $l$-Lipschitz gradients: for all $\theta_1$ and $\theta_2$, $\norm{\nabla J(\theta_1)- \nabla J(\theta_2)}\leq l\norm{\theta_1-\theta_2}$.
	\item $\rho$-Lipschitz Hessians: for all $\theta_1$ and $\theta_2$, $\norm{\nabla^2 J(\theta_1)- \nabla^2 J(\theta_2)}\leq \rho\norm{\theta_1-\theta_2}$.
\end{itemize}
\textbf{Assumption 2}: $g(\theta;\tau)$ and $\Hess(\theta;\tau)$ should follow,
\begin{align*}
    \norm{g(\theta ; \tau)\!-\!\nabla J(\theta)} \le M_1\textit{ and } \norm{\Hess(\theta ; \tau)\!-\!\nabla^2\! J(\theta)} \le M_2,
\end{align*}

The bounds in Proposition \ref{lipschitz-prop} implies assumption 1, while Lemma \ref{lemma:gradHessBounds} implies assumption 2.  Invoking Theorem 1 of  \cite{jordan2018stochastic} leads to following result that establishes convergence of Algorithm \ref{alg:approx_policyNewton} to an $\epsilon$-SOSP. 

\begin{theorem}[\textit{\textbf{High-probability bound for approximate solution}}]
	\label{thm:approxNewtonBound}
	Assume \ref{ass:bddRewards}--\ref{ass:liphesspolicy}. Fix $\delta\in (0,1]$. Let
	$m_k\geq\max(\frac{M_1}{c_1\epsilon},\frac{M_1^2}{c_3^2\epsilon^2})\log\Big(\frac{d \sqrt{\rho}}{\epsilon^{1.5}\delta' c_3}\Big)$ and $b_k\geq\max(\frac{M_2}{c_4\sqrt{\rho\epsilon}},\frac{M_2^2}{c_4^2\rho\epsilon})\log\Big(\frac{d \sqrt{\rho}}{\epsilon^{1.5}\delta' c_4}\Big)$, where $c_1,c_2=\frac{1}{300}$ and $c_3,c_4=\frac{1}{200}$.
 Set $\hat{N}=\frac{1}{\sqrt{\epsilon}}$.
	
	Then Algorithm \ref{alg:approx_policyNewton} will output an $\epsilon-$SOSP of $J$ with probability at least $1-\delta'$ within
	\begin{equation}
		O\bigg(\frac{\sqrt{\rho}\chi}{\epsilon^{1.5}}\bigg(\max\bigg(\frac{M_1}{\epsilon},\frac{M_1^2}{\epsilon^2}\bigg)+\max\bigg(\frac{M_2}{\sqrt{\rho\epsilon}},\frac{M_2^2}{\rho\epsilon}\bigg)\frac{1}{\sqrt{\epsilon}}\bigg)\bigg)
	\end{equation}
	total stochastic gradient and Hessian-vector product sample complexity where $\chi\geq J(\theta_0)-J^*.$
\end{theorem}

\begin{remark}\label{sample complexity}
	Under the conditions of Theorem \ref{thm:approxNewtonBound} and $\epsilon\leq \min\{\frac{M_1}{c_3},\frac{M_2^2}{c_4^2\rho}\}$, ACR-PN will output an $\epsilon$-SOSP with probability at least $1-\delta'$ within $O(\epsilon^{-1.5})$ iterations, maximum $O(\epsilon^{-3.5})$ gradient sample complexity and $O(\epsilon^{-3})$ Hessian–vector
	product sample complexity.
\end{remark}

We can conclude from Remark \ref{sample complexity}  that Algorithm \ref{alg:approx_policyNewton} finds an $\epsilon$-SOSP after $\bar N = O(\epsilon^{-1.5})$ iterations. The total number of trajectories required for gradient averaging and Hessian product averaging is $O(\frac{M_1^2}{\epsilon^{7/2}})$ and $O(\frac{M_2^2}{\rho\epsilon^{5/2}})$, respectively when $\epsilon$ is small. These numbers are comparable to those of Algorithm \ref{alg:policyNewton}, where the cubic-regularized problem was solved exactly.

 \section{Convergence proofs}
 \label{sec:proofs}
\subsection{Proof of Proposition \ref{lipschitz-prop}}\label{pf:lipschitz-prop}
\begin{lemma}\label{lipgrad}
Under Assumptions \ref{ass:bddRewards} and \ref{ass:paramreg}, for any parameter $\theta$ and trajectory $\tau$, we have 
\begin{align*}
    &\norm{\nabla \Phi(\theta ; \tau)} \le G_g,\quad 
    \norm{\nabla^2 \Phi(\theta ; \tau)} \le L_1KH^2 \\
    &\norm{g(\theta ; \tau)} \le G_g, \quad \textrm{and} \quad
    \norm{\Hess(\theta ; \tau)} \le G_{\Hess},
\end{align*}
where $G_g=GKH^2$, and $G_{\Hess} :=  H^3 G^2 K + L_1K H^2.$
\end{lemma}
\begin{proof}
Using the definition of $\Phi(\theta ;\tau)$, we have
\begin{align}
    \norm{\nabla \Phi(\theta ; \tau)} &= \norm{\sum_{i=0}^{H-1} \Psi_{i}(\tau) \nabla \log \pi(a_i | s_i ; \theta)} \le \sum_{i=0}^{H-1} |\Psi_{i}(\tau)| \cdot \norm{\nabla \log \pi(a_i | s_i ; \theta)} 
    \le  G \sum_{i=0}^{H-1} |\Psi_{i}(\tau)|.
\end{align}
We can establish a bound on $|\Psi_{i}(\tau)|$ as follows:
\begin{align}
    |\Psi_{i}(\tau)| = |\sum_{h=i}^{H-1} \gamma^{h-1} c(s_h, a_h) | \le K \sum_{h=i}^{H-1} \gamma^{h-1} \le KH .
\end{align}
implying
\begin{align}
    \norm{\nabla \Phi(\theta ; \tau)} \le GKH^2..
\end{align}
Similarly,
\begin{align}
    \norm{\nabla^2 \Phi(\theta ; \tau)} &= \norm{\sum_{i=0}^{H-1} \Psi_{i}(\tau) \nabla^2 \log \pi(a_i | s_i ; \theta)} \\
    &\le \sum_{i=0}^{H-1} |\Psi_{i}(\tau)| \norm{\nabla^2 \log \pi(a_i | s_i ; \theta)} \\
    &\le  L_1 \sum_{i=0}^{H-1} |\Psi_{i}(\tau)| \le L_1KH^2. 
\end{align}
It is now easy to show that the gradient estimate $g(\theta ; \tau)$ is bounded as follows:
\begin{align}
    \norm{g(\theta ; \tau)} = \norm{\nabla \Phi(\theta; \tau)} \le GKH^2 = G_g .
\end{align}
Next, we show that the Hessian estimate $\Hess(\theta ; \tau)$ is bounded. Notice that
\begin{align}
    \norm{\Hess(\theta ; \tau)} &= \norm{\nabla \Phi(\theta ; \tau) \nabla^{\top} \log p(\tau ;\theta) + \nabla^2 \Phi(\theta ; \tau)} \\
    &\le \norm{\nabla \Phi(\theta ; \tau)}  \norm{\nabla \log p(\tau ;\theta)} + \norm{\nabla^2 \Phi(\theta ; \tau)} \\
    &\le GKH^2 \norm{\nabla \log p(\tau ;\theta)} + L_1KH^2 .
\end{align}
Using the relation $\nabla \log p(\tau ; \theta) = \sum_{h=0}^{H-1} \nabla \log \pi (a_h | s_h ; \theta)$, we have
\begin{align}
    \norm{\nabla \log p(\tau ; \theta)} \le \sum_{h=0}^{H-1} \norm{\nabla \log \pi (a_h | s_h ; \theta)} \le HG.
\end{align}
Therefore, we obtain
\begin{align}
    \norm{\Hess(\theta ; \tau)} \le H^3 G^2 K + L_1K H^2 = G_{\Hess}.
\end{align}
Hence proved.
\end{proof}
From the above lemma, one can easily interpret that the objective function, i.e., $J(\theta)$ and its gradient, $\nabla J(\theta)$ are Lipschitz continuous. We now need to show that the Hessian of the objective, i.e., $\nabla^2 J(\theta)$ is Lipschitz.
\begin{lemma}\label{liphess}
Under Assumptions \ref{ass:bddRewards}, \ref{ass:paramreg} and \ref{ass:liphesspolicy}, we have for any $(\theta_1, \theta_2)$,
\begin{align}\label{eq:liphess}
    &\norm{\nabla^2 J(\theta_1) - \nabla^2 J(\theta_2)} \le L_{\Hess} \norm{\theta_1 - \theta_2},\textrm{ where }\\
&
    L_{\Hess} :=H^4 G^3 K + 3H^3 G L_1 K + L_2KH^2 .
\end{align}
\end{lemma}
\begin{proof}
We begin with the expression for the Hessian of our objective, i.e.,
\begin{align}
    \nabla^2 J(\theta) = \sum_{\tau} (\nabla \Phi(\theta ; \tau) \nabla^{\top} p(\tau ; \theta) + p(\tau ; \theta) \nabla^2 \Phi(\theta ; \tau)) \, .
\end{align}
Notice that
\begin{align}
    \nabla^2 J(\theta_1) - \nabla^2 J(\theta_2) &= 
    \sum_{\tau} (\nabla \Phi(\theta_1 ; \tau) \nabla^{\top} p(\tau ; \theta_1) + p(\tau ; \theta_1) \nabla^2 \Phi(\theta_1 ; \tau)) \, \\
    &-
    \sum_{\tau} (\nabla \Phi(\theta_2 ; \tau) \nabla^{\top} p(\tau ; \theta_2) + p(\tau ; \theta_2) \nabla^2 \Phi(\theta_2 ; \tau)) \,  \\
    &= 
    \sum_{\tau} \left( \nabla \Phi(\theta_1 ; \tau) \nabla^{\top} p(\tau ; \theta_1) - \nabla \Phi(\theta_2 ; \tau) \nabla^{\top} p(\tau ; \theta_2) \right) \,  \\
    &+
    \sum_{\tau} \left( p(\tau ; \theta_1) \nabla^2 \Phi(\theta_1 ; \tau) - p(\tau ; \theta_2) \nabla^2 \Phi(\theta_2 ; \tau) \right) \, \\
    \textrm{Hence, } \norm{\nabla^2 J(\theta_1) - \nabla^2 J(\theta_2)} 
    &\le
    \sum_{\tau} \norm{\nabla \Phi(\theta_1 ; \tau) \nabla^{\top} p(\tau ; \theta_1) - \nabla \Phi(\theta_2 ; \tau) \nabla^{\top} p(\tau ; \theta_2)} \, \\
    &+
    \sum_{\tau} \norm{ p(\tau ; \theta_1) \nabla^2 \Phi(\theta_1 ; \tau) - p(\tau ; \theta_2) \nabla^2 \Phi(\theta_2 ; \tau) } \,  . \label{eq:int}
\end{align}

For ease of notation, let $\Phi_j := \Phi(\theta_j ; \tau)$ and $p_j := p(\tau ; \theta_j)$. Considering the first summand in \eqref{eq:int}
\begin{align}
    \norm{\nabla \Phi_1 \nabla^{\top} p_1 - \nabla \Phi_2 \nabla^{\top} p_2}
    &\le \norm{\nabla \Phi_1 \nabla^{\top} p_1 - \nabla \Phi_1 \nabla^{\top} p_2} +
    \norm{\nabla \Phi_1 \nabla^{\top} p_2 - \nabla \Phi_2 \nabla^{\top} p_2} \\
    &\le  \norm{\nabla \Phi_1} \norm{\nabla p_1 - \nabla p_2} + \norm{\nabla p_2} \norm{\nabla \Phi_1 - \nabla \Phi_2} . \label{eq:int1}
\end{align}
Using the mean-value theorem for vector-valued functions, we have
\begin{align}
    \nabla p_1 - \nabla p_2 = \nabla^2 p_{h_1} (\theta_1 - \theta_2),
\end{align}
where $p_{h_1} = p(\tau ; \theta_{h_1})$, and $\theta_{h_1} = (1-h_1)\theta_1 + h_1 \theta_2$ for some $h_1 \in [0, 1]$. Therefore,
\begin{align}
    \norm{\nabla p_1 - \nabla p_2} &\le \norm{\nabla^2 p_{h_1}} \norm{\theta_1 - \theta_2} \\
    &= \norm{\nabla (p_{h_1} \nabla \log p_{h_1})} \norm{\theta_1 - \theta_2} \\
    &= \norm{\nabla p_{h_1} \nabla^{\top} \log p_{h_1} + p_{h_1} \nabla^2 \log p_{h_1}} \norm{\theta_1 - \theta_2} \\
    &= \norm{p_{h_1}(\nabla\log p_{h_1} \nabla^{\top} \log p_{h_1} + \nabla^2 \log p_{h_1})} \norm{\theta_1 - \theta_2} \\    
    &\stackrel{(a)}{\leq}  p_{h_1} \left( \norm{\nabla \log p_{h_1}}^2 + \norm{\nabla^2 \log p_{h_1}} \right) \norm{\theta_1 - \theta_2} \\
    &\le p_{h_1} \left( H^2 G^2 + H L_1 \right) \norm{\theta_1 - \theta_2} , \label{eq:int1_1}
\end{align}
where (a) uses the relation $\nabla \log p_j = \sum_{h=0}^{H-1} \nabla \log \pi (a_h | s_h ; \theta_j)$ and $\nabla^2 \log p_j = \sum_{h=0}^{H-1} \nabla^2 \log \pi (a_h | s_h ; \theta_j)$

Similarly,
\begin{align}
    \norm{\nabla \Phi_1 - \nabla \Phi_2} &\le \norm{\nabla^2 \Phi_{h_2}} \norm{\theta_1 - \theta_2} , \label{eq:int1_2}
\end{align}
where $\Phi_{h_2} = \Phi(\theta_{h_2} ; \tau)$, and $\theta_{h_2} = (1-h_2)\theta_1 + h_2 \theta_2$ for some $h_2 \in [0, 1]$. Plugging \eqref{eq:int1_1} and \eqref{eq:int1_2} in \eqref{eq:int1}, we have
\begin{align}\label{lip-bound}
    \norm{\nabla \Phi_1 \nabla^{\top} p_1 - \nabla \Phi_2 \nabla^{\top} p_2} 
    &\le
    p_{h_1} \left( H^2 G^2 + H L_1 \right) \norm{\nabla \Phi_1} \norm{\theta_1 - \theta_2} + p_2 \norm{ \nabla \log p_2} \norm{\nabla^2 \Phi_{h_2}} \norm{\theta_1 - \theta_2} \\
    &\le
    p_{h_1} (H^4 G^3 K + H^3 L_1 G K) \norm{\theta_1 - \theta_2} + p_2  H^3 G L_1 K \norm{\theta_1 - \theta_2} .
\end{align}
Where the final inequality used the bounds obtained in Lemma \ref{lipgrad}. Summing both the sides of \eqref{lip-bound} over $\tau$, we obtain, 
\begin{align}
    \sum_{\tau} \norm{\nabla \Phi_1 \nabla^{\top} p_1 - \nabla \Phi_2 \nabla^{\top} p_2} \, 
    &\le H^4 G^3 K + 2 H^3 G L_1 K \norm{\theta_1 - \theta_2} .
\end{align}
In the above inequality, we used the fact that $\sum_{\tau} p(\tau ; \theta) \, = 1$ for all $\theta$. Moving on to the second summand in \eqref{eq:int}, we have
\begin{align}
    \norm{p_1 \nabla^2 \Phi_1 - p_2 \nabla^2 \Phi_2} &\le
    \norm{p_1 \nabla^2 \Phi_1 - p_1 \nabla^2 \Phi_2} + \norm{p_1 \nabla^2 \Phi_2 - p_2 \nabla^2 \Phi_2} , \\
    &\le p_1 \norm{\nabla^2 \Phi_1 - \nabla^2 \Phi_2} + \norm{\nabla^2 \Phi_2} |p_1 - p_2|. \label{eq:int2}
\end{align}
Using the mean value theorem,
\begin{align}
    |p_1 - p_2| &\le \norm{\nabla p_{h_3}}  \norm{\theta_1 - \theta_2} \\
    &\le p_{h_3} \norm{\nabla \log p_{h_3}}  \norm{\theta_1 - \theta_2} \\
    &\le  p_{h_3} (HG )  \norm{\theta_1 - \theta_2} .
\end{align}
Considering the first term in \eqref{eq:int2},
\begin{align}
    \norm{\nabla^2 \Phi_1 - \nabla^2 \Phi_2}
    &= \norm{\sum_{i=0}^{H-1} \Psi_{i}(\tau) \left( \nabla^2 \log \pi(a_i | s_i ; \theta_1) - \nabla^2 \log \pi(a_i | s_i ; \theta_2) \right)} \\
    &\le \sum_{i=0}^{H-1} |\Psi_{i}(\tau)|  \norm{ \nabla^2 \log \pi(a_i | s_i ; \theta_1) - \nabla^2 \log \pi(a_i | s_i ; \theta_2) } \\
    &\le L_2 \norm{\theta_1 - \theta_2} \sum_{i=0}^{H-1} KH \le L_2KH^2 \norm{\theta_1 - \theta_2}.
\end{align}
We used here \ref{ass:liphesspolicy} in the second last inequality above. Plugging the above results in \eqref{eq:int2}, we have
\begin{align}
    \norm{p_1 \nabla^2 \Phi_1 - p_2 \nabla^2 \Phi_2} &\le
    p_1 L_2KH^2  \norm{\theta_1 - \theta_2} + p_{h_3} H^3GL_1K  \norm{\theta_1 - \theta_2} .
\end{align}
summing on both sides, we obtain
\begin{align}
    \sum_{\tau} \norm{p_1 \nabla^2 \Phi_1 - p_2 \nabla^2 \Phi_2} \,  &\le
    (L_2KH^2+H^3GL_1K)  \norm{\theta_1 - \theta_2} .
\end{align}
Therefore, by substituting in the original equation, we obtain
\begin{align}
    \norm{\nabla^2 J(\theta_1) - \nabla^2 J(\theta_2)} &\le ( H^4 G^3 K + 2 H^3 G L_1 K) \norm{\theta_1 - \theta_2} + (L_2KH^2+H^3GL_1K)  \norm{\theta_1 - \theta_2}  \\
    &= (H^4 G^3 K + 3H^3 G L_1 K + L_2KH^2) \norm{\theta_1 - \theta_2} .
\end{align}
Hence, proved.
\end{proof}

\begin{remark}
From \eqref{eq:liphess}, it can be easily seen that
\begin{gather}
    \norm{\nabla J(\theta_1) - \nabla J(\theta_2) - \nabla^2 J(\theta_2)(\theta_1 - \theta_2)} \le \frac{L_{\Hess}}{2} \norm{\theta_1 - \theta_2}^2 ,\\
    |J(\theta_1) - J(\theta_2) - \innerproduct{\nabla J(\theta_2)}{\theta_1 - \theta_2} - \frac{1}{2} \innerproduct{\theta_1 - \theta_2}{\nabla^2 J(\theta_2)(\theta_1 - \theta_2)}| \le \frac{L_{\Hess}}{6} \norm{\theta_1 - \theta_2}^3 .
\end{gather}
\end{remark}
\subsection{Proof of Theorem \ref{thm:policyGradAndHessian}}\label{pf:policyGradAndHessian}
\label{appendix:policyGradStuff}
\begin{proof}
The result os available in \cite{shen2019hessian} and we provide the proof here for the sake of completeness. Re-writing the objective function \eqref{J} as follows:
\begin{align}
    J(\theta) &:= \E[\tau \sim p(\tau ; \theta)]{\mathcal{G}(\tau)} 
    =  \E[\tau \sim p(\tau ; \theta)] {\sum_{h=0}^{H-1} \gamma^{h-1} c(s_h, a_h) }
    = \sum_{h=0}^{H-1} \E[\tau_h \sim p(\tau_h ; \theta)] {\gamma^{h-1} c(s_h, a_h) } .
\end{align}
The last equality above holds as the term inside the expectation is independent of future events, i.e., the trajectory $(s_{0:h}, a_{0:h})$ does not depend on the trajectory $(s_{h+1:H-1}, a_{h+1:H-1})$. 
Replacing the expectation by a summation over all trajectories
\begin{align}
    J(\theta) = \sum_{h=0}^{H-1} \sum_{\tau_h} \gamma^{h-1} c(s_h, a_h) p(\tau_h ; \theta) \,  .
\end{align}
Differentiating on both sides, we obtain
\begin{equation}
    \nabla J(\theta) = \sum_{h=0}^{H-1} \sum_{\tau_h} \gamma^{h-1} c(s_h, a_h) \nabla p(\tau_h ; \theta) \, .
\end{equation}
Using now the fact that $\nabla p(\tau_h ; \theta) = p(\tau_h ; \theta)\nabla \log p(\tau_h ; \theta)$, we obtain
\begin{align}
    \nabla J(\theta) &= \sum_{h=0}^{H-1} \sum_{\tau_h} \gamma^{h-1} c(s_h, a_h) \nabla  \log p(\tau_h ; \theta) \, p(\tau_h ; \theta) \, \\
    &= \sum_{h=0}^{H-1} \E[\tau_h \sim p(\tau_h ; \theta)] {\gamma^{h-1} c(s_h, a_h) \nabla  \log p(\tau_h ; \theta) } .
\end{align}
From \eqref{traj}, we can show that $\nabla \log p(\tau ; \theta) = \sum_{h=0}^{H-1} \nabla \log \pi (a_h | s_h ; \theta)$, and thus
\begin{align}
    \nabla J(\theta) 
    &= \sum_{h=0}^{H-1} \E[\tau_h \sim p(\tau_h ; \theta)]{\gamma^{h-1} c(s_h, a_h) \sum_{i=1}^h \nabla \log \pi (a_i | s_i ; \theta) } \\
    &= \sum_{h=0}^{H-1} \sum_{i=0}^h  \E[\tau_h \sim p(\tau_h ; \theta)]{\gamma^{h-1} c(s_h, a_h)  \nabla \log \pi (a_i | s_i ; \theta)} \\
    &= \sum_{h=0}^{H-1} \sum_{i=0}^h  \E[\tau \sim p(\tau ; \theta)] {\gamma^{h-1} c(s_h, a_h)  \nabla \log \pi (a_i | s_i ; \theta) } .
\end{align}
where in the last equality we use the fact that $ \gamma^{h-1} c(s_h, a_h)  \nabla \log \pi (a_i | s_i ; \theta) $ with $i \le h$ is independent of the randomness after $a_h$. Interchanging the order of summation, we obtain
\begin{align}
    \nabla J(\theta) &= \sum_{i=0}^{H-1} \sum_{h=i}^{H-1}  \E[\tau \sim p(\tau ; \theta)] {\gamma^{h-1} c(s_h, a_h)  \nabla \log  \pi (a_i | s_i ; \theta) } \\
    &= \sum_{i=0}^{H-1} \E[\tau \sim p(\tau ; \theta)]{\left( \sum_{h=i}^{H-1} \gamma^{h-1} c(s_h, a_h) \right) \nabla \log \pi (a_i | s_i ; \theta) } \\
    &= \sum_{i=0}^{H-1} \E[\tau \sim p(\tau ; \theta)]{\Psi_{i}(\tau) \nabla \log \pi (a_i | s_i ; \theta) } . 
\end{align}
This concludes the proof of the first claim.
For the second claim, notice that 
\begin{align}
    \nabla^2 J(\theta) &= \nabla \left( \sum_{\tau} \nabla \Phi(\theta ; \tau) p(\tau ; \theta) \, \right) \\
    &= \sum_{\tau} (\nabla \Phi(\theta ; \tau) \nabla^{\top} p(\tau ; \theta) + \nabla^2 \Phi(\theta ; \tau) p(\tau ; \theta) \,  \\
    &= \sum_{\tau} \left( \nabla \Phi(\theta ; \tau) \nabla^{\top} \log p(\tau ; \theta) + \nabla^2 \Phi(\theta ; \tau) \right) p(\tau ; \theta)) \, \\
    &= \E[\tau \sim p(\tau ; \theta)]{\nabla \Phi(\theta ; \tau) \nabla^{\top} \log p(\tau ; \theta) + \nabla^2 \Phi(\theta ; \tau) } .
\end{align}
Hence, proved.
\end{proof}

\subsection{Proof of Theorem \ref{thm:cubRegNewtonBound}}\label{pf:cubRegNewtonBound}
\label{appendix:cubRegNewtonBoundProof}
The proof proceeds through a sequence of lemmas while following the technique from \cite{kkumar2018zeroth}. However, unlike the aforementioned reference, we operate in an RL framework and more importantly, with  unbiased gradient and Hessian estimates, leading to a major deviation in the proof as compared to \cite{kkumar2018zeroth}.

\begin{lemma}
Let $\Bar{\theta} = \argmax_{x \in \mathbb{R}^d} \Tilde{J}(x, \theta, \Hess, g, \alpha)$. Then, we have
\begin{align}\label{np}
    &g + \Hess(\Bar{\theta}-\theta) + \frac{\alpha}{2}\norm{\Bar{\theta}-\theta}(\Bar{\theta}-\theta) = 0 , \\
    &\Hess + \frac{\alpha}{2}\norm{\Bar{\theta}-\theta} I_d \succeq 0 .
\end{align}
where $I_d$ is the identity matrix.
\end{lemma}
\begin{proof}
See Lemma 4.3 from \cite{kkumar2018zeroth} and \cite{np06}.
\end{proof}

We now derive the second and third-order error bounds on our Hessian estimate.
\begin{lemma}\label{lem:estimatebounds}
Let $\Bar{g}_k$ and $\Bar{\Hess}_k$ be computed as in Algorithm \ref{alg:policyNewton}, and assume $b_k \ge 4(1 + 2\log 2d)$. 

Then we have
\begin{align}\label{error}
    \E{\norm{\Bar{g}_k - \nabla J(\theta_{k-1})}^2} \le \frac{G_g^2}{m_k} , \qquad
    \E{ \norm{\Bar{\Hess}_k - \nabla^2 J(\theta_{k-1})}^3 } \le \frac{4 \sqrt{15 (1 + 2\log 2d)} d G_{\Hess}^3}{ b_k^\frac{3}{2}}.
\end{align}
\end{lemma}
\begin{proof}
Using the fact that the estimate $\Bar{g}_k$ is unbiased, we have

\begin{align*}
    \E{\norm{\Bar{g}_k - \nabla J(\theta_{k-1})}^2} &= \E{ \norm{ \frac{1}{m_k} \sum_{\tau \in \mathcal{T}_m} \left( g(\theta_{k-1} ; \tau) - \nabla J(\theta_{k-1})\right)}^2 } \\
    &= \frac{1}{m_k^2} \sum_{\tau}\E{{\norm{ g(\theta_{k-1} ; \tau) - \nabla J(\theta_{k-1})}^2}} \\
    &+ 
    \frac{1}{m_k^2} \sum_{\tau \ne \tau'} \E{\innerproduct{g(\theta_{k-1} ; \tau) - \nabla J(\theta_{k-1})}{g(\theta_{k-1} ; \tau') - \nabla J(\theta_{k-1})}}, \\
    &\le \frac{1}{m_k^2} \E{\sum_{\tau \in \mathcal{T}_m} \norm{ g(\theta_{k-1} ; \tau)}^2} \le  \frac{G_g^2}{m_k}, 
\end{align*}
where the second summand in the second equation equating to zero arrives from the fact that the trajectories are independent.A
This establishes the first bound in \eqref{error}.
Now we turn to proving the second bound in \eqref{error}.
By Theorem 1 in \cite{tro16}, we have
\begin{align}\label{eqtemp}
    \E{ \norm{\Bar{\Hess}_k - \nabla^2 J(\theta_{k-1})}^2} \le
    \frac{2 C(d)}{b_k^2} \left( \norm{\sum_{\tau \in \mathcal{T}_b} \E{\Delta_{k, \tau}^2}} + C(d) \E{\max_{\tau} \norm{\Delta_{k, \tau}}^2} \right) ,
\end{align}
where $\Delta_{k, \tau} = \Hess(\theta_{k-1} ; \tau) - \nabla^2 J(\theta_{k-1})$ and $C(d) = 4(1 + 2\log 2d)$. It is easy to see that
\begin{align}
    \E{\norm{\Delta_{k, \tau}}^2} &\le \E{\norm{\Hess(\theta_{k-1} ; \tau)}^2} \le G_{\Hess}^2 , \quad \textrm{and} \label{eqtemp1} \\
    \norm{\sum_{\tau \in \mathcal{T}_b} \E{\Delta_{k, \tau}^2}} &\le \sum_{\tau \in \mathcal{T}_b} \norm{\E{\Delta_{k, \tau}^2}} \le \sum_{\tau \in \mathcal{T}_b} \E{\norm{\Delta_{k, \tau}}^2} \label{eqtemp2}.
\end{align}
Using \eqref{eqtemp1} and \eqref{eqtemp2} in \eqref{eqtemp}, we obtain
\begin{align*}
    \E{ \norm{\Bar{\Hess}_k - \nabla^2 J(\theta_{k-1})}^2} &\le
    \frac{2 C(d)}{b_k^2} \left( b_k G_{\Hess}^2 + C(d) G_{\Hess}^2 \right) \le \frac{4 C(d)}{b_k} G_{\Hess}^2,
\end{align*}
where in the last inequality we use the assumption that $b_k \ge C(d)$. Let $\norm{}_F$ denote the Frobenius' norm and using Holder's inequality, we obtain
\begin{align}\label{temp2}
    \E{\norm{\Bar{\Hess}_k - \nabla^2 J(\theta_{k-1})}^3} 
    &\le \E{\norm{\Bar{\Hess}_k - \nabla^2 J(\theta_{k-1})} \cdot \norm{\Bar{\Hess}_k - \nabla^2 J(\theta_{k-1})}^2_F} \\
    &\le \left( \E{\norm{\Bar{\Hess}_k - \nabla^2 J(\theta_{k-1})}^2} \cdot \E{\norm{\Bar{\Hess}_k - \nabla^2 J(\theta_{k-1})}^4_F} \right)^{\frac{1}{2}} .
\end{align}
Note that $\Bar{\Hess}_k - \nabla^2 J(\theta_{k-1}) = \frac{1}{b_k} \sum_{\tau \in \mathcal{T}_b} \Delta_{k, \tau}$, therefore we have
\begin{align*}
    \E{\norm{\Bar{\Hess}_k - \nabla^2 J(\theta_{k-1})}^4_F} = \E{\norm{\frac{1}{b_k} \sum_{\tau \in \mathcal{T}_b} \Delta_{k, \tau}}^4_F} = \frac{1}{b_k^4} \E{\norm{ \sum_{\tau \in \mathcal{T}_b} \Delta_{k, \tau}}^4_F} \le \frac{3 \E{\norm{\Delta_{k, \tau}}^4_F}}{b_k^2} ,
\end{align*}
where the final inequality comes from Rosenthal's inequality (see Lemma \ref{rosenthal} in Appendix \ref{appendix:probineq}).
Using the fact that $\norm{\cdot}_F \le \sqrt{d} \norm{\cdot}$ and the inequality from Lemma \ref{rvbound} in Appendix \ref{appendix:probineq}, we have
\begin{align*}
    \E{\norm{\Bar{\Hess}_k - \nabla^2 J(\theta_{k-1})}^4_F} &\le \frac{3 d^2 \E{\norm{\Delta_{k, \tau}}^4}}{b_k^2}
    \le \frac{15 d^2 \E{\norm{\Hess(\theta_{k-1} ; \tau_i)}^4}}{b_k^2}
    \le \frac{15 d^2 G_{\Hess}^4}{b_k^2},
\end{align*}
which when combined in \eqref{temp2} leads to the second bound in  \eqref{error}.
\end{proof}

We next state a result that will be used in a subsequent lemma.
\begin{lemma}
\label{lem:lambdamaxineq}
If for any two matrices $A$ and $B$, and a scalar $c$, we have
\begin{equation}
    A \preceq B + c I,
\end{equation}
where $I$ is the identity matrix of the appropriate dimension, then the following holds: \begin{equation}
    c \ge \lambda_{max} (A) - \norm{B}
\end{equation}
\end{lemma}
\begin{proof}
See Appendix \ref{appendix:lambdamaxproof}.
\end{proof}

\begin{lemma}\label{lemmalhs}
Let $\{ \theta_k \}$ be computed by Algorithm \ref{alg:policyNewton}. Then, we have
\begin{align}\label{lhs}
    &\sqrt{\E{\norm{\theta_k - \theta_{k-1}}^2}} \\
    &\ge
    \max \left\{ \sqrt{\frac{\E{\norm{\nabla J(\theta_k)}} - \delta_k^g -\delta_k^{\Hess}}{L_{\Hess} + \alpha_K}},
    \frac{-2}{\alpha_k + 2 L_{\Hess}} \left[ \E{\lambda_{\min} \left( \nabla^2 J(\theta_k)\right)} + \sqrt{2(\alpha_k + L_{\Hess}) \delta^{\Hess}_k} \right]
    \right\} ,
\end{align}
where $\delta_k^g, \delta_k^{\Hess} > 0$ are chosen such that
\begin{align}\label{deltas}
    \E{\norm{\nabla J(\theta_{k-1}) - \Bar{g}_k}^2} \le \left( \delta_k^g \right)^2, \quad \textrm{and} \quad
    \E{\norm{\nabla^2 J(\theta_{k-1}) - \Bar{\Hess}_k}^3} \le \left( 2(L_{\Hess} + \alpha_k) \delta_k^{\Hess} \right)^{\frac{3}{2}} .
\end{align}
\end{lemma}
\begin{proof}
Firstly, note that $\delta_k^g$ and $\delta_k^{\Hess}$ are inversely proportional to $\sqrt{m_k}$ and $b_k$, respectively and are therefore well-defined. Now, by the equality condition in Lemma \ref{np}, we have
\begin{align}\label{eq:*}
    \norm{\nabla J(\theta_{k})} &\le
    \norm{\nabla J(\theta_{k}) - \nabla J(\theta_{k-1}) - \nabla^2 J(\theta_{k-1}) (\theta_k - \theta_{k-1})} + \norm{\nabla J(\theta_{k-1}) - \Bar{g}_k} \\
    &+ \norm{\nabla^2 J(\theta_{k-1}) - \Bar{\Hess}_k} \norm{\theta_k - \theta_{k-1}} + \frac{\alpha_k}{2} \norm{\theta_k - \theta_{k-1}}^2 \\
    &\le \frac{(L_{\Hess} + \alpha_k)}{2} \norm{\theta_k - \theta_{k-1}}^2 + \norm{\nabla J(\theta_{k-1}) - \Bar{g}_k} + \norm{\nabla^2 J(\theta_{k-1}) - \Bar{\Hess}_k} \norm{\theta_k - \theta_{k-1}} \\
    &\le (L_{\Hess} + \alpha_k) \norm{\theta_k - \theta_{k-1}}^2 + \norm{\nabla J(\theta_{k-1}) - \Bar{g}_k} + \frac{\norm{\nabla^2 J(\theta_{k-1}) - \Bar{\Hess}_k}^2}{2(L_{\Hess} + \alpha_k)} ,
\end{align}
where we used Young's inequality. In the last step, we take expectation on both sides and use the relations in \eqref{deltas} to obtain
\begin{align}\label{fosp}
    \frac{(\E{\norm{\nabla J(\theta_k)} - \delta^g_k - \delta^{\Hess}_k})}{L_{\Hess} + \alpha_k} \le \E{\norm{\theta_k - \theta_{k-1}}^2} .
\end{align}
By the inequality in Lemma \ref{np}, and the smoothness result in Lemma \ref{liphess}, we have
\begin{align*}
    \nabla^2 J(\theta_k) &\succeq \nabla^2 J(\theta_{k-1}) - L_{\Hess} \norm{\theta_k - \theta_{k-1}} I_d = \nabla^2 J(\theta_{k-1}) - \Bar{\Hess}_k + \Bar{\Hess}_k - L_{\Hess} \norm{\theta_k - \theta_{k-1}} I_d \\
    &\succeq \nabla^2 J(\theta_{k-1}) - \Bar{\Hess}_k - \frac{(\alpha_k + 2 L_{\Hess}) \norm{\theta_k - \theta_{k-1}}}{2} I_d,
\end{align*}
which implies that
\begin{align}\label{eq:**}
    \frac{(\alpha_k + 2 L_{\Hess}) \norm{\theta_k - \theta_{k-1}}}{2}
    &\ge \lambda_{min}(\nabla^2 J(\theta_{k-1}) - \Bar{\Hess}_k) - \lambda_{\min} \left( \nabla^2 J(\theta_k) \right)
\end{align}
Taking expectations on both sides, and using the definition of $\delta^{\Hess}_k$ in \eqref{deltas}, we have
\begin{align}\label{sosp}
    \sqrt{\E{\norm{\theta_k - \theta_{k-1}}^2}} &\ge \E{\norm{\theta_k - \theta_{k-1}}} \\ &\ge
    \frac{-2}{\alpha_k + 2 L_{\Hess}} \left[ \E{\lambda_{\min} \left( \nabla^2 J(\theta_k)\right)} + \sqrt{2(\alpha_k + L_{\Hess})) \delta^{\Hess}_k}   \right] .
\end{align}
Combining the above inequality with \eqref{fosp}, we obtain \eqref{lhs}.
\end{proof}

\begin{lemma}\label{lemmarhs}
Let $\{ \theta_k \}$ be computed by Algorithm [\ref{alg:policyNewton}] for a given iteration limit $N \ge 1$, we have
\begin{align}\label{rhs}
    &\E{\norm{\theta_R - \theta_{R-1}}^3} \\
    &\le \frac{36}{\sum_{k=1}^N \alpha_k} \left[  J(\theta_{0}) - J^* + \sum_{k=1}^N \frac{4 \left( \delta_k^g \right)^\frac{3}{2}}{\sqrt{3 \alpha_k}} + \sum_{k=1}^N \left( \frac{18\sqrt[4]{2}}{\alpha_k} \right)^2 \left( (L_{\Hess} + \alpha_k) \delta^{\Hess}_k \right)^{\frac{3}{2}} \right] ,
\end{align}
where $R$ is a random variable whose probability distribution $P_R(\cdot)$ is supported on $\{1, \ldots, N\}$ and given by
\begin{align}\label{pdfR}
    P_R(R=k) = \frac{\alpha_k}{\sum_{k=1}^N \alpha_k}, \qquad k = 1, \ldots , N , 
\end{align}
and $\delta_k^g, \delta_k^{\Hess} > 0$ are defined as before in \eqref{deltas}.
\end{lemma}
\begin{proof}
We can see that by Lemma \ref{liphess}, \eqref{aux} and the fact that $\alpha_k \ge L_{\Hess}$, we have
\begin{align}\label{eq:ineq1}
    J(\theta_k) &\le J(\theta_{k-1}) + \Tilde{J}^k(\theta_k) + \norm{\nabla J(\theta_{k-1} - \Bar{g}_k)} \norm{\theta_k - \theta_{k-1}} 
    + \frac{1}{2} \norm{\nabla^2 J(\theta_{k-1}) - \Bar{\Hess}_k} \norm{\theta_k - \theta_{k-1}}^2 .
\end{align}
Moreover, by Lemma \ref{np}, we have
\begin{align}\label{eq:ineq2}
    \Tilde{J}^k(\theta_k) = -\frac{1}{2} \innerproduct{\Bar{\Hess}_k (\theta_k - \theta_{k-1})}{(\theta_k - \theta_{k-1})} - \frac{\alpha_k}{3} \norm{\theta_k - \theta_{k-1}}^3 \le -  \frac{\alpha_k}{12} \norm{\theta_k - \theta_{k-1}}^3 .
\end{align}
Combining \eqref{eq:ineq1} and \eqref{eq:ineq2}, we obtain
\begin{align}
    \frac{\alpha_k}{12} \norm{\theta_{k-1} - \theta_{k}}^3  
    &\le J(\theta_{k-1}) - J(\theta_{k}) + \norm{\nabla J(\theta_{k-1} - \Bar{g}_k)} \norm{\theta_k - \theta_{k-1}} \\
    &+ \frac{1}{2} \norm{\nabla^2 J(\theta_{k-1}) - \Bar{\Hess}_k} \norm{\theta_k - \theta_{k-1}}^2 \\
    &\le J(\theta_{k-1}) - J(\theta_{k}) + \frac{4}{\sqrt{3 \alpha_k}} \norm{\nabla J(\theta_{k-1} - \Bar{g}_k)}^{\frac{3}{2}} \\
    &+ \left( \frac{9\sqrt{2}}{\alpha_k} \right)^2 \norm{\nabla^2 J(\theta_{k-1} - \Bar{\Hess}_k)}^3 + \frac{\alpha_k}{18} \norm{\theta_k - \theta_{k-1}}^3 , \label{eq:ineq3}
\end{align}
where the last inequality follows from the fact $ab \le \frac{a^p}{\lambda^p p} + \frac{\lambda^q b^q}{q}$ for $p, q$ satisfying $\frac{1}{p} + \frac{1}{q} = 1$ and $\lambda >0$. 

We now take expectation on both sides of \eqref{eq:ineq3}  and use \eqref{deltas} to obtain
\begin{align}
    \frac{\alpha_k}{36} \E{\norm{\theta_k - \theta_{k-1}}^3} \le
    J(\theta_{k-1}) - J(\theta_{k}) + \frac{4 \left( \delta_k^g \right)^\frac{3}{2}}{\sqrt{3 \alpha_k}} + \left( \frac{18\sqrt[4]{2}}{\alpha_k} \right)^2 \left( (L_{\Hess} + \alpha_k) \delta^{\Hess}_k \right)^{\frac{3}{2}}.
\end{align}
Summing over $k=1, \ldots, N$, dividing both sides by $\sum_{k=1}^N \alpha_k$ and noting \eqref{pdfR}, we obtain the bound in \eqref{rhs}.
\end{proof}

\subsection*{Proof of Theorem \ref{thm:cubRegNewtonBound}}
\begin{proof}

First, note that by \eqref{params}, Lemma \ref{error}, we can ensure that \eqref{deltas} is satisfied by $\delta^g_k = 2\epsilon / 5$ and $\delta^{\Hess}_k = \epsilon / 144 $. Moreover, by Lemma \ref{lemmarhs}, we have
\begin{align}
    \E{\norm{\theta_R - \theta_{R-1}}^3} &\le \frac{12}{L_{\Hess}} \left[ \frac{J(\theta_0) - J^*}{N} + \frac{4\left( 2/5 \right)^{\frac{3}{2}}}{3 \sqrt{L_{\Hess}}} \epsilon^{\frac{3}{2}} + \frac{18^2\sqrt{2}}{9 \cdot 6^3 \sqrt{L_{\Hess}}} \epsilon^{\frac{3}{2}}  \right] \\
    &\le \frac{1}{L_{\Hess}^{\frac{3}{2}}} \left[ \frac{12\sqrt{L_{\Hess}}(J(\theta_0) - J^*)}{N} + 6.88 \epsilon^{\frac{3}{2}} \right] \nonumber\\
    &\le \frac{8 \epsilon^{\frac{3}{2}}}{L_{\Hess}^{\frac{3}{2}}}. \label{eq:n13}
\end{align}
The inequality in \eqref{eq:n13} follows by substituting the value of $N$ specified in the theorem statement. Furthermore, from Lemma \ref{lemmalhs} and using Lyapunov inequality i.e.,
\begin{equation*}
    \begin{split}
       \bigg[\E{\norm{\theta_R - \theta_{R-1}}^2}\bigg]^{1/2}\leq \bigg[\E{\norm{\theta_R - \theta_{R-1}}^3}\bigg]^{1/3}\leq \frac{{2} \epsilon^{\frac{1}{2}}}{L_{\Hess}^{\frac{1}{2}}}
    \end{split} .
\end{equation*}
Using the bound above in conjunction with \eqref{fosp} and \eqref{sosp}, we obtain
\begin{align*}
    \sqrt{\E{\norm{\nabla J(\theta_k)}}} \le \sqrt{\left(16 + \frac{2}{5} + \frac{1}{144} \right) \epsilon} \le 5 \sqrt{\epsilon} ,
\end{align*}
and
\begin{align*}
    \frac{\E{-\lambda_{\min} \left( \nabla^2 J(\theta_k) \right)}}{\sqrt{L_{\Hess}}} \le \left( 7 - \frac{1}{6} \right) \sqrt{\epsilon} \le 7 \sqrt{\epsilon}.
\end{align*}
The main result in \eqref{main} follows from the two inequalities above.

Finally, note that the total number of required samples to obtain such a solution is bounded by
\begin{align*}
    \sum_{k=1}^N m_k = O \left( \frac{1}{\epsilon^{\frac{7}{2}}} \right) , \qquad
    \sum_{k=1}^N b_k = O \left( \frac{d^{\frac{2}{3}}}{\epsilon^{\frac{5}{2}}} \right) .
\end{align*}
\end{proof}
\subsection{Proof of Lemma \ref{lemma:gradHessBounds}}\label{pf:gradHessBounds}
\begin{proof}
Recall that $g(\theta ; \tau)=\nabla\Phi(\theta,\tau)$. By Lemma \ref{lipgrad} we have $\norm{\nabla \Phi(\theta ; \tau)} \le GKH^2$. Further, from \eqref{eq:ghlh} $\norm{\nabla J(\theta)}\le KGH^3$. Hence,
\begin{align*}
    \norm{g(\theta ; \tau)-\nabla J(\theta)}&\leq\norm{g(\theta ; \tau)}+\norm{\nabla J(\theta)}\\
    &\leq GKH^2 + KGH^3\\
    &=GKH^2(H+1)=M_1.
\end{align*}
Squaring and taking expectations, we obtain 
\begin{align*}
    \mathbb{E}\norm{g(\theta ; \tau)-\nabla J(\theta)}^2\leq M_1^2.
\end{align*}
Next, we establish bounds on the Hessian estimate. Note that $ \norm{\Hess(\theta ; \tau)} \le H^3G^2K + L_1KH^2 = G_{\Hess}$. Further from \eqref{eq:ghlh} we have $\norm{\nabla^2 J(\theta)}\le G_{\Hess}$. Hence,
\begin{align*}
    \norm{\Hess(\theta,\tau)-\nabla^2J(\theta))}\le 2G_{\Hess}=M_2,\\
    \mathbb{E}\norm{\Hess(\theta,\tau)-\nabla^2J(\theta))}^2\le M_2^2.\\
\end{align*}
\end{proof}
\subsection{Proof of Theorem \ref{thm:High-probability bound}}\label{pf:High-probability bound}

\begin{lemma}\label{concentration bound}
Let $m_k = \max\bigg(\frac{M_1}{t},\frac{M_1^2}{t^2}\bigg)\frac{8}{3}\log\frac{2d}{\delta'} , b_k = \max\bigg(\frac{M_2}{\sqrt{t_1}},\frac{M_2^2}{t_1}\bigg)\frac{8}{3}\log\frac{2d}{\delta'}$
any positive constants and $\delta'\in (0,1)$. Then, with probability $1-\delta'$ we have
\begin{align}\label{var-grad}
        \norm{\Bar{g}_k-\nabla J(\theta_k)}^2\le t^2, \textit{ and }
      \norm{\Bar{\Hess}_k-\nabla^2J(\theta)}^3\le t_1^{\frac{3}{2}}.
\end{align}
\end{lemma}
\begin{proof}
Following \cite[equation 2.2.8]{Tropp}, we define the matrix variance statistic of a random matrix $Z$ as
\begin{align*}
    v(Z) &= \max\{ \Vert\textbf{Var}_1(Z)\Vert,\Vert\textbf{Var}_2(Z)\Vert\}, \textrm{ where }\\
    \textbf{Var}_1(Z) &= \mathbb{E}(Z-\mathbb{E}(Z))(Z-\mathbb{E}(Z))^T, \textrm{ and}\\ \textbf{Var}_2(Z) &= \mathbb{E}(Z-\mathbb{E}(Z))^T(Z-\mathbb{E}(Z)).
\end{align*}
Letting  $\nabla\Tilde{\Phi}(\theta;\tau) = \nabla\Phi(\theta;\tau)-\nabla J(\theta)$, we have the following expression for the centered gradient estimate:
\begin{align}
    \Tilde{g}_k=\frac{1}{m_k}\sum_{\tau\in\mathcal{T}_m}\big(\nabla\Tilde{\Phi}(\theta;\tau)\big).
\end{align}
Using the triangle inequality and Jensen’s inequality, the matrix variance $v(\Tilde{g}_k)$ is simplified bounded as follows:
\begin{equation*}
    \begin{split}
        v(\Tilde{g}_k) &= \frac{1}{m_k}\max\Big\{\Big\lVert\mathbb{E}\sum_{\tau\in\mathcal{T}_m} \nabla\Tilde{\Phi}(\theta;\tau)\nabla\Tilde{\Phi}(\theta;\tau)^T\Big\rVert ,\Big\lVert\mathbb{E}\sum_{\tau\in\mathcal{T}_{m_k}} \nabla\Tilde{\Phi}(\theta;\tau)^T\nabla\Tilde{\Phi}(\theta;\tau)\Big\rVert \Big\},\\
        &\le \frac{1}{m_k}\max\Big\{\mathbb{E}\sum_{\tau\in\mathcal{T}_m} \Big\lVert\nabla\Tilde{\Phi}(\theta;\tau)\nabla\Tilde{\Phi}(\theta;\tau)^T\Big\rVert ,\mathbb{E}\sum_{\tau\in\mathcal{T}_{m_k}} \Big\lVert\nabla\Tilde{\Phi}(\theta;\tau)^T\nabla\Tilde{\Phi}(\theta;\tau)\Big\rVert \Big\}\leq\frac{M_1^2}{m_k}.
    \end{split}
\end{equation*}
From an application of matrix Bernstein inequality, see \cite[Theorem 7.3.1]{Tropp}, we obtain
\begin{gather}
  \mathbb{P}[\norm{\Bar{g}_k-\nabla J(\theta)}\geq t]\leq 2d\exp\bigg(-\frac{t^2/2}{v(\Tilde{g}_k)+M_1t/(3m_k)}\bigg)\leq 2d\exp\bigg(-\frac{3m_k}{8}\min\bigg\{\frac{t}{M_1}\frac{t^2}{M_1^2}\bigg\}\bigg).
  \end{gather}
Thus, for $m_k\geq\max\bigg(\frac{M_1}{t},\frac{M_1^2}{t^2}\bigg)\frac{8}{3}\log\frac{2d}{\delta'}$, we have 
\[ \norm{\Bar{g}_k-\nabla J(\theta)}\leq t   \textrm{ with probability }  1-\delta'.\]
The first claim follows.

Next, we turn to proving the second claim concerning the Hessian estimate $\Bar{\Hess}_k$.
As in the case of the high-probability bound for the gradient estimate above, we define $\Tilde{\Hess}(\theta;\tau) = \Hess(\theta;\tau)-\nabla^2 J(\theta)$, and the centered Hessian $\Tilde{\Hess}_k=\frac{1}{b_k}\sum_{\tau \in \mathcal{T}_{b}}\Tilde{\Hess}(\theta;\tau)$. 
The variance of $\Tilde{\Hess}_k$ can be bounded as follows:
\begin{equation*}
    \begin{split}
        v[\Tilde{\Hess}_k]=\frac{1}{b_k^2}\norm{\sum_{\tau \in \mathcal{T}_{b_k}}\mathbb{E}\Big[\big(\Tilde{\Hess}(\theta;\tau)\big)^2\big]}\leq\frac{M_2^2}{b_k}
    \end{split}.
\end{equation*}
Applying the matrix Bernstein inequality for the centered Hessian leads to the following bound:
\begin{gather}\label{hess_ineq}
  \mathbb{P}[\norm{\Bar{\Hess}_k-\nabla^2 J(\theta)}\geq t']\leq 2d\exp\bigg(-\frac{3b_k}{8}\min\bigg\{\frac{t'}{M_2}\frac{t'^2}{M_2^2}\bigg\}\bigg)
 \end{gather}
Thus, for $b_k\geq\max\bigg(\frac{M_2}{t'},\frac{M_2^2}{t'^2}\bigg)\frac{8}{3}\log\frac{2d}{\delta'}$, we have 
\[\norm{\Bar{\Hess}_k-\nabla^2 J(\theta)}\leq t'   \textrm{ with probability }  1-\delta'. \]
The claim concerning $\Bar\Hess_k$ follows by setting $t'=\sqrt{t_1}$.
\end{proof}
\begin{lemma}\label{lemmalhs1}
Let $\{ \theta_k \}$ be computed by Algorithm \ref{alg:policyNewton}. Then with,$m_k,b_k$ as in Lemma \ref{concentration bound}, we have
\begin{equation}\label{lhs1}
    \begin{split}
        {\norm{\theta_k - \theta_{k-1}}}
        &\ge \max \left\{ \sqrt{\frac{{\norm{\nabla J(\theta_k)}} - t -t_1}{L_{\Hess} + \alpha_K}},
    \frac{-2}{\alpha_k + 2 L_{\Hess}} \left[ {\lambda_{\min} \left( \nabla^2 J(\theta_k)\right)} + \sqrt{2(\alpha_k + L_{\Hess}) t_1} \right]
    \right\},
    \end{split}
\end{equation}
with probability $1-2\delta'$. 
\end{lemma}
\begin{proof}
We first recall \eqref{eq:*} from the proof of Lemma \ref{lemmalhs} below.
\begin{align*}
    \norm{\nabla J(\theta_{k})} 
    &\le (L_{\Hess} + \alpha_k) \norm{\theta_k - \theta_{k-1}}^2 + \norm{\nabla J(\theta_{k-1}) - \Bar{g}_k} + \frac{\norm{\nabla^2 J(\theta_{k-1}) - \Bar{\Hess}_k}^2}{2(L_{\Hess} + \alpha_k)} ,
\end{align*}
 From Lemma \ref{concentration bound}, with probability $1-\delta'$, we have
\begin{align}\label{deltas1}
    {\norm{\nabla J(\theta_{k-1}) - \Bar{g}_k}^2} \le  t^2, \qquad
    {\norm{\nabla^2 J(\theta_{k-1}) - \Bar{\Hess}_k}^3} \le \left( 2(L_{\Hess} + \alpha_k) t_1 \right)^{\frac{3}{2}} .
\end{align}
Thus, with probability $1-2\delta'$,
\begin{align}\label{fosp1}
     \sqrt{\frac{({\norm{\nabla J(\theta_k)} - t - t_1})}{L_{\Hess} + \alpha_k}} \le {\norm{\theta_k - \theta_{k-1}}}.
\end{align}
Recall that \eqref{eq:**} from the proof of Lemma \ref{lemmalhs} established the following inequality:
\begin{align*}
    \frac{(\alpha_k + 2 L_{\Hess}) \norm{\theta_k - \theta_{k-1}}}{2}
    &\ge \lambda_{min}(\nabla^2 J(\theta_{k-1}) - \Bar{\Hess}_k) - \lambda_{\min} \left( \nabla^2 J(\theta_k) \right).
\end{align*}
Using the bounds from \eqref{deltas1} in the inequality above, we obtain
\begin{align*}
    {\norm{\theta_k - \theta_{k-1}}} &\ge
    \frac{-2}{\alpha_k + 2 L_{\Hess}} \left[ {\lambda_{\min} \left( \nabla^2 J(\theta_k)\right)} + \sqrt{2(\alpha_k + L_{\Hess})) t_1}   \right] .
\end{align*}
Combining the above inequality with \eqref{fosp1}, we obtain \eqref{lhs1}. Hence proved.
\end{proof}

\begin{lemma}\label{lemmarhs1}
Let $\{ \theta_k \}$ be computed by Algorithm [\ref{alg:policyNewton}] for a given iteration limit $N \ge 1$. Then under the setting of Lemma \ref{concentration bound} we have
\begin{align}\label{rhs1}
    &{\norm{\theta_R - \theta_{R-1}}^3} \\
    &\le \frac{36}{\sum_{k=1}^N \alpha_k} \left[  J(\theta_{0}) - J^* + \sum_{k=1}^N \frac{4t^\frac{3}{2}}{\sqrt{3 \alpha_k}} + \sum_{k=1}^N \left( \frac{18\sqrt[4]{2}}{\alpha_k} \right)^2 \left( (L_{\Hess} + \alpha_k) t_1 \right)^{\frac{3}{2}} \right] ,
\end{align}
with probability $1-2\delta'N$ and $R$ is a random variable
with distribution specified in Lemma \ref{lemmarhs}
\end{lemma}
\begin{proof}
From  \eqref{eq:ineq1},\eqref{eq:ineq2} and \eqref{eq:ineq3} in the proof of Lemma \ref{lemmarhs}, we have
\begin{align}
    \frac{\alpha_k}{12} \norm{\theta_{k-1} - \theta_{k}}^3  
    &\le J(\theta_{k-1}) - J(\theta_{k}) + \frac{4}{\sqrt{3 \alpha_k}} \norm{\nabla J(\theta_{k-1} - \Bar{g}_k)}^{\frac{3}{2}} 
    + \left( \frac{9\sqrt{2}}{\alpha_k} \right)^2 \norm{\nabla^2 J(\theta_{k-1} - \Bar{\Hess}_k)}^3 \\&\qquad+ \frac{\alpha_k}{18} \norm{\theta_k - \theta_{k-1}}^3. \label{eq:ineq3_1}
\end{align}
Rearranging the terms above and using the bounds from \eqref{deltas1}, we obtain
\begin{align}
    \frac{\alpha_k}{36} {\norm{\theta_k - \theta_{k-1}}^3} \le
    J(\theta_{k-1}) - J(\theta_{k}) + \frac{4t^\frac{3}{2}}{\sqrt{3 \alpha_k}} + \left( \frac{18\sqrt[4]{2}}{\alpha_k} \right)^2 \left( (L_{\Hess} + \alpha_k) t_1 \right)^{\frac{3}{2}}.
\end{align}
Summing over $k=1, \ldots, N$, dividing both sides by $\sum_{k=1}^N \alpha_k$, and noting the fact that after $N$-th iteration of Algorithm \ref{alg:policyNewton}, the concentration bounds  in Lemma \ref{concentration bound} hold with probability $1-2\delta'N$, we obtain the bound in \eqref{rhs1} with probability $1-2\delta'N$.
\end{proof}

\subsection*{Proof of Theorem \ref{thm:High-probability bound}}
\label{appendix:cubRegNewtonBoundProof1}
\begin{proof}
First, note that by \eqref{params1}, Lemma \ref{concentration bound}, we can ensure that \eqref{deltas1} is satisfied by $t = 2\epsilon / 5$ and $t_1 = \epsilon / 144 $. Moreover, by Lemma \ref{lemmarhs1} and \eqref{eq:n13} with probability $1-2\delta'N$, we have 
\begin{align*}
    {\norm{\theta_R - \theta_{R-1}}^3} & \le \frac{8 \epsilon^{\frac{3}{2}}}{L_{\Hess}^{\frac{3}{2}}} ,
\end{align*}
where we choose $N$ according to \eqref{params1}. 

From Lemma \ref{lemmalhs1}, with probability $1-2\delta'N$, we have
\begin{align*}
    \sqrt{{\norm{\nabla J(\theta_R)}}} \le \sqrt{\left(16 + \frac{2}{5} + \frac{1}{144} \right) \epsilon} \le 5 \sqrt{\epsilon} \textit{ and } \frac{{-\lambda_{\min} \left( \nabla^2 J(\theta_R) \right)}}{\sqrt{L_{\Hess}}} \le \left( 7 - \frac{\sqrt{2}}{6} \right) \sqrt{\epsilon} \le 7 \sqrt{\epsilon}.
\end{align*}
Thus, \eqref{main1} follows implying $\theta_R$ is a $\epsilon$-SOSP with high-probability.
\end{proof}

\section{Conclusions}\label{conclusion}
In this paper, we proposed policy Newton algorithms with cubic regularization. Our algorithms form unbiased estimates of the gradient as well as the Hessian of the value function using sample trajectories. From a rigorous convergence analysis, we established that our policy Newton algorithms converge to a second-order stationary point (SOSP) of the value function, which implies the algorithms avoid saddle points. Further, the sample complexity of our algorithms to find an $\epsilon$-SOSP is $O(\epsilon^{-3.5})$, and this result is an improvement over the $O(\epsilon^{-4.5})$ bound currently available in the literature.
\appendix

\bibliography{arxiv/ref}

\section{Proof of Lemma \ref{lem:lambdamaxineq}}
\label{appendix:lambdamaxproof}
We state and prove a useful result that will imply bound in Lemma \ref{lem:lambdamaxineq}.
\begin{lemma}{\label{lambdamax}}
For any square matrix $A \in \mathbb{R}^{d \times d}$ and for all vectors $v \in \mathbb{R}^d$, the following holds
\begin{equation}\label{cond1}
    v^{\top} A v \le \lambda \norm{v}^2 ,
\end{equation}
for some $\lambda \in \mathbb{R}$, if and only if
\begin{equation}\label{cond2}
    \lambda_{\max} (A) \le \lambda .
\end{equation}
\end{lemma}
\begin{proof}
We shall first prove the forward argument which is quite trivial. We define the map
\begin{equation}
    \lambda_v (A) := \frac{v^{\top} A v }{\norm{v}^2}
\end{equation}
Now, given \eqref{cond1}, we shall find a $v_*$ such that $A v_* = \lambda_{max}(A) v_*$, i.e. $v_*$ is the eigenvector associated with the largest eigenvalue of $A$. As $v_* \in \mathcal{C}(A) \subseteq \mathbb{R}^d$, the following should hold
\begin{equation}
    \lambda_{v_*} (A) \le \lambda .
\end{equation}
But,
\begin{equation}
    \lambda_{v_*} (A) = \frac{v_*^{\top} A v_* }{\norm{v_*}^2} = \lambda_{\max}(A) \frac{v_*^{\top} v_* }{\norm{v_*}^2} = \lambda_{\max}(A) \le \lambda.
\end{equation}
Now to check if the converse holds, we start by arguing that $\lambda_v (A) \le \lambda_{\max}(A)$ for all $v$ and $A$. We argue that $\lambda_{v} (A) \in [\lambda_{\min}(A), \lambda_{\max}(A)]$, as it has the form
\begin{equation}
    \lambda_{v} (A) = \frac{\sum_{i=1}^r \lambda_i a_i^2}{\sum_{i=1}^r a_i^2},
\end{equation}
where $r$ is the rank of $A$ and ${a_i}$ are the coefficients of $v$. Hence, $\lambda_{v} (A)$ can be thought of as a weighted average of all eigenvalues of $A$. Therefore, given \eqref{cond2}, we have for all $v\in \mathbb{R}^d$,
\begin{align}
    \lambda_{v} \le \lambda ,
\end{align}
which satisfies \eqref{cond1}.
\end{proof}

\begin{proof}\textit{(Lemma \ref{lem:lambdamaxineq})}
For all $v \in \mathbb{R}^d$, we have
\begin{align}
    v^{\top} A v &\le v^{\top} B v + c \norm{v}^2 \\
    &\le \norm{B} \cdot \norm{v}^2 + c \norm{v}^2 \\
    &= (\norm{B} + c) \norm{v}^2,
\end{align}
where in the second line, we used the Cauchy-Schwartz inequality. Now by using, Lemma \ref{lambdamax} in the last line, we obtain
\begin{gather}
    \lambda_{max}(A) \le \norm{B} + c, \quad \textrm{implying} \quad c \ge \lambda_{max} (A) - \norm{B}.
\end{gather}
\end{proof}

\section{A few probabilistic inequalities}
\label{appendix:probineq}
We state and prove two probabilistic inequalites, which are used in the proof of Theorem \ref{thm:cubRegNewtonBound}. In particular, the result below as well as  Rosenthal's inequality (stated in Lemma \ref{rosenthal}) are used in the proof of Lemma \ref{lem:estimatebounds}
\begin{lemma}\label{rvbound}
Let $Z \in \mathbb{R}^{d \times d}$ be a random matrix. Then, we have
\begin{align*}
    \E{\norm{Z - \E{Z}}^4} \le 5 \E {\norm{Z}^4} .
\end{align*}
\end{lemma}
\begin{proof}
We can re-write the expectation as
\begin{align*}
    \E{\norm{Z - \E{Z}}^4} = \Var{\norm{Z - \E{Z}}^2} + \left( \E{ \norm{Z - \E{Z}}^2 } \right)^2 .
\end{align*}
Consider the first term
\begin{align*}
    \Var { \norm{Z - \E{Z}}^2 }
    &= \Var { \norm{Z}^2 + \norm{\E{Z}}^2 - 2 \innerproduct{Z}{\E{Z}} } \\
    &= \Var { \norm{Z}^2 - 2 \innerproduct{Z}{\E{Z}}} \qquad \left( \because \Var{ \norm{\E{Z}}^2} = 0 \right) \\
    &\le \Var { \norm{Z}^2 } + 4 \Var { \innerproduct{Z}{\E{Z}}} + 4 \sqrt{\Var{\norm{Z}^2}} \sqrt{\Var{\innerproduct{ Z}{\E{Z}}}} .
\end{align*}
Now for the second term
\begin{align*}
    \left( \E{ \norm{Z - \E{Z}}^2 } \right)^2 &= \left( \E{\norm{Z}^2} - \norm{\E{Z}}^2 \right)^2 \\
    &=  \left( \E{ \norm{Z}^2 } \right)^2 + \norm{\E{Z}}^4 - 2\E{\norm{Z}^2} \norm{\E{Z}}^2 .
\end{align*}
Simplifying the terms under the root
\begin{align*}
    \sqrt{\Var{\norm{Z}^2}} &= \sqrt{\E{ \norm{Z}^4 } - \left( \E{ \norm{Z}^2 } \right)^2} \\
    &= \sqrt{\E{ \norm{Z}^4 }} \sqrt{  1 - \frac{\left( \E{ \norm{Z}^2 } \right)^2}{\E{ \norm{Z}^4 }} } \\
    &\le \sqrt{\E{ \norm{Z}^4 }} \left( 1 - \frac{\left( \E{ \norm{Z}^2 } \right)^2}{2 \E{ \norm{Z}^4 }} \right) \\
\end{align*}
where in the last inequality we used the fact that $\sqrt{1-x} \le 1 - \frac{x}{2}$ .
\begin{align*}
    \Var {\innerproduct{ Z}{\E{Z}}} &= \E{ \innerproduct{Z}{\E{Z}}^2 } - \left( \E{ \innerproduct{Z}{\E{Z}} } \right)^2 \\
    &\le \E{ \norm{Z}^2 \norm{\E{Z}}^2 } - \norm{\E{Z}}^4 \\
    &\le \E{ \norm{Z}^2 } \norm{\E{ Z }}^2  .
\end{align*}
Putting these results together
\begin{align*}
    \E{ \norm{Z - \E{Z}}^4 } &\le 
    \Var { \norm{Z}^2 } + \left( \E{ \norm{Z}^2 } \right)^2 \\
    &+ 4 \Var { \innerproduct{Z}{\E{Z}} } + \norm{\E{Z}}^4 - 2\E{ \norm{Z}^2 } \norm{\E{Z}}^2 \\
    &+ 4 \sqrt{\Var{\norm{Z}^2}} \sqrt{\Var{\innerproduct{ Z}{\E{Z}}}} \\
    &\le \E{ \norm{Z}^4 } + 2\E{ \norm{Z}^2 } \norm{\E{Z}}^2 - 3\norm{\E{Z}}^4 \\
    &+ 4\sqrt{\E{ \norm{Z}^4 }} \left( 1 - \frac{\left( \E{ \norm{Z}^2 } \right)^2}{2 \E{ \norm{Z}^4 }} \right) \sqrt{\E{ \norm{Z}^2 } \norm{\E{ Z }}^2} .
\end{align*}
Note that by Jensen's inequality, we have $\E{ \norm{Z}^2 } \norm{\E{ Z }}^2 \le \left( \E{ \norm{Z}^2 } \right)^2 \le \E{ \norm{Z}^4 }$. Substituting these results above and further simplification, we have
\begin{align*}
    \E{ \norm{Z - \E{Z}}^4 } \le
    5 \E{ \norm{Z}^4 } - 3\norm{\E{Z}}^4 \le 5 \E{ \norm{Z}^4 } .
\end{align*}
\end{proof}

\begin{lemma}[Rosenthal's inequality]
\label{rosenthal}
Let $\{ X_1, \ldots, X_n \}$ be a sequence of random $d \times d$ square matrices with $\E{X_i}=0$ for all $i$. Then the following inequality holds 
\begin{align*}
    \E{\norm{\sum_{i=1}^n X_i}^4_F} \le 3 n^2 \E{\norm{X_i}^4_F} ,
\end{align*}
where $\norm{\cdot}_F$ denotes the Frobenius norm of a matrix.
\end{lemma}
\begin{proof}
We start by using the definition of variance as follows
\begin{align}\label{step1}
    \E{\norm{\sum_{i=1}^n X_i}^4_F} = \Var{\norm{\sum_{i=1}^n X_i}^2_F} + \left( \E{\norm{\sum_{i=1}^n X_i}^2_F} \right)^2 .
\end{align}
Consider the first term in \eqref{step1}
\begin{align*}
    \Var{\norm{\sum_{i=1}^n X_i}^2_F} &= 
    \Var{\Tr{\sum_{i} X_i^{\top} \sum_{j} X_j}} \\
    &= \Var{\Tr{\sum_{i} X_i^{\top} X_i + 2 \sum_{i<j} X_i^{\top} X_j}} \\
    &= \Var{\sum_{i} \norm{X_i}^2_F + 2 \sum_{i<j} \Tr{X_i^{\top} X_j}} \\
    &= \sum_{i} \Var{\norm{X_i}^2_F} + 4 \sum_{i<j} \Var{\Tr{X_i^{\top} X_j}} .
\end{align*}
Expanding the terms under summation
\begin{align*}
    \sum_{i} \Var{\norm{X_i}^2_F} &= \sum_{i} \E{\norm{X_i}^4_F} - \sum_{i} \left( \E{\norm{X_i}^2_F} \right)^2 ,\\
    \sum_{i<j} \Var{\Tr{X_i^{\top} X_j}} &= \sum_{i<j} \E{\left( \Tr{X_i^{\top} X_j}\right)^2} - \sum_{i<j} \left( \E{\Tr{X_i^{\top} X_j}} \right)^2 \\
    &= \sum_{i<j} \E{\left( \Tr{X_i^{\top} X_j}\right)^2}. \qquad (\because \E{\Tr{X_i^{\top} X_j}} = 0 \textrm{ for } i \neq j)
\end{align*}
Therefore,
\begin{align*}
    \Var{\norm{\sum_{i=1}^n X_i}^2_F}
    &= \sum_{i} \E{\norm{X_i}^4_F} + 4\sum_{i<j} \E{\left( \Tr{X_i^{\top} X_j}\right)^2} - \sum_{i} \left( \E{\norm{X_i}^2_F} \right)^2  \\
    &= \sum_{i} \E{\norm{X_i}^4_F} + 2\sum_{i \neq j} \E{\left( \Tr{X_i^{\top} X_j}\right)^2} - \sum_{i} \left( \E{\norm{X_i}^2_F} \right)^2  \\
    &\le 2 \sum_{i} \sum_{j} \E{\left(\Tr{X_i^{\top}X_j}\right)^2} \le 2 \sum_{i} \sum_{j} \E{\left(\Tr{X_i^{\top}X_i}\right)^2} \\
    &= 2 n^2 \E{\norm{X_i}^4_F} .
\end{align*}
In the second last line we used the property of inner products, i.e. $\innerproduct{X_i}{X_j} \le \innerproduct{X_i}{X_i} = \norm{X_i}^2$ for all $(i, j)$ pairs. Now taking the second term in \eqref{step1}
\begin{align*}
    \E{\norm{\sum_{i=1}^n X_i}^2_F}
    &= \E{\Tr{\sum_{i} X_i^{\top} \sum_{j} X_j }} \\
    &= \Tr{\sum_{i}\sum_{j} \E{X_i^{\top} X_j }} \\
    &= \Tr{\sum_{i} \E{X_i^{\top} X_i}} = \sum_{i} \E{\norm{X_i}^2_F} = n \E{\norm{X_i}^2_F} .
\end{align*}
Plugging these results in \eqref{step1}
\begin{align*}
    \E{\norm{\sum_{i=1}^n X_i}^4_F} &\le 2 n^2 \E{\norm{X_i}^4_F} + n^2 \left(\E{\norm{X_i}^2_F} \right)^2 \\
    &\le 2 n^2 \E{\norm{X_i}^4_F} + n^2 \E{\norm{X_i}^4_F} \qquad (\textrm{Jensen's inequality.})\\
    &= 3 n^2 \E{\norm{X_i}^4_F} .
\end{align*}
\end{proof}

\end{document}